\documentclass{article}

\usepackage{microtype}
\usepackage{graphicx}
\usepackage{subfigure}
\usepackage{booktabs} 
\usepackage{amsmath}
\usepackage{tikz-cd}
\usepackage{amsthm}
\usepackage{amssymb}
\usepackage{bm}
\usepackage{caption}

\usepackage{hyperref}

\newcommand{\kn}{{\it ker}}
\newcommand{\chn}{{\it chn}}
\renewcommand{\sp}{{\it sp}}

\DeclareMathOperator{\spn}{span}
\DeclareMathOperator{\diag}{diag}
\DeclareMathOperator\erf{erf}

\renewcommand{\theequation}{\arabic{section}.\arabic{equation}} 

\newtheorem{theorem}{Theorem}[section]
\newtheorem{lemma}[theorem]{Lemma}

\newtheorem{definition}{Definition}[section]
\newtheorem{corollary}[theorem]{Corollary}

\definecolor{darkgreen}{rgb}{0,0.6,0}
\definecolor{darkred}{rgb}{0.7,0.0,0}
\definecolor{darkblue}{rgb}{0,0.0,0.6}

\newcommand{\jcom}[1]{\textcolor{darkgreen}{}}
\newcommand{\xcom}[1]{\textcolor{blue}{}}

\newcommand{\jp}[1]{\textcolor{darkred}{}}
\newcommand{\sss}[1]{\textcolor{orange}{}}

\usepackage[accepted]{icml2018}

\icmltitlerunning{Mean Field Theory of Convolutional Neural Networks}

\begin{document}

\twocolumn[
\icmltitle
{
Dynamical Isometry and a Mean Field Theory of CNNs:\\
How to Train 10,000-Layer Vanilla Convolutional Neural Networks
}



\icmlsetsymbol{equal}{*}

\begin{icmlauthorlist}
\icmlauthor{Lechao Xiao}{brain,gbr}
\icmlauthor{Yasaman Bahri}{brain,gbr}
\icmlauthor{Jascha Sohl-Dickstein}{brain}
\icmlauthor{Samuel S. Schoenholz}{brain}
\icmlauthor{Jeffrey Pennington}{brain}
\end{icmlauthorlist}

\icmlaffiliation{brain}{Google Brain}
\icmlaffiliation{gbr}{Work done as part of the Google AI Residency program (g.co/airesidency)}

\icmlcorrespondingauthor{Lechao Xiao}{xlc@google.com}

\vskip 0.3in
]



\printAffiliationsAndNotice{}  

\begin{abstract}

In recent years, state-of-the-art methods in computer vision have utilized increasingly deep convolutional neural network architectures (CNNs), with some of the most successful models employing hundreds or even thousands of layers. A variety of pathologies such as vanishing/exploding gradients make training such deep networks challenging. While residual connections and batch normalization do enable training at these depths, it has remained unclear whether such specialized architecture designs are truly necessary to train deep CNNs. In this work, we demonstrate that it is possible to train vanilla CNNs with ten thousand layers or more simply by using an appropriate initialization scheme. We derive this initialization scheme theoretically by developing a mean field theory for signal propagation and by characterizing the conditions for \emph{dynamical isometry}, the equilibration of singular values of the input-output Jacobian matrix. These conditions require that the convolution operator be an orthogonal transformation in the sense that it is norm-preserving. We present an algorithm for generating such random initial orthogonal convolution kernels and demonstrate empirically that they enable efficient training of extremely deep architectures.
\end{abstract}

\section{Introduction}
\label{intro}

Deep convolutional neural networks (CNNs) have been crucial to the success of deep learning. Architectures based on CNNs have achieved unprecedented accuracy in domains ranging across computer vision~\cite{krizhevsky2012}, speech recognition~\cite{hinton2012speech}, natural language processing~\cite{collobert2011natural, kalchbrenner2014convolutional, kim2014convolutional}, and recently even the board game Go~\cite{silver2016,silver2017mastering}.

    \begin{figure}[t!]
    \begin{center}
    \vspace{0.1cm}
    \centerline{
    \hspace{-0.2cm}\includegraphics[width=0.86\columnwidth]{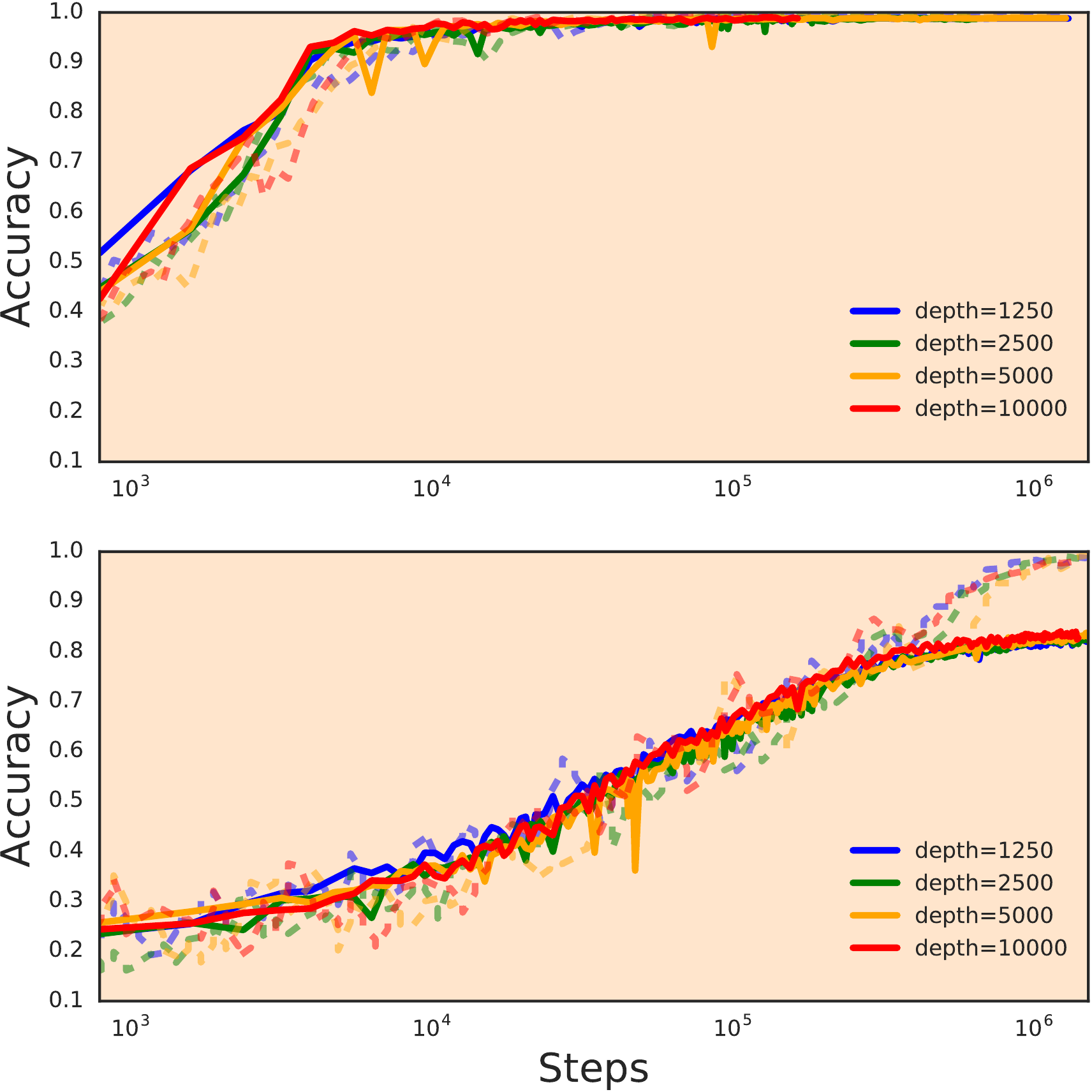}}
    \caption{
    Extremely deep CNNs can be trained without the use of batch normalization or residual connections simply by using a Delta-Orthogonal initialization with critical weight and bias variance and appropriate (in this case, $\operatorname{tanh}$) nonlinearity.  
    Test (solid) and training (dashed) curves on MNIST (top) and CIFAR-10 (bottom) for depths $1{,}250$, $2{,}500$, $5{,}000$, and $10,{000}$.}
    \label{fig:cm10000}
    \end{center}
    \vskip -0.37in
    \end{figure}
The performance of deep convolutional networks has improved as these networks have been made ever deeper. For example, some of the best-performing models on ImageNet~\cite{imagenet_cvpr09} have employed hundreds or even a thousand layers~\cite{he2016deep,he2016identity}. However, these extremely deep architectures have been trainable only in conjunction with techniques like residual connections~\cite{he2016deep} and batch normalization~\cite{ioffe2015batch}. It is an open question whether these techniques qualitatively improve model performance or whether they are necessary crutches that solely make the networks easier to train. In this work, we study vanilla CNNs using a combination of theory and experiment to disentangle the notions of trainability and generalization performance. In doing so, we show that through a careful, theoretically-motivated initialization scheme, we can train vanilla CNNs with 10,000 layers using no architectural tricks.

Recent work has used mean field theory to build a theoretical understanding of neural networks with random parameters~\citep{poole2016,schoenholz2016,yang2017,schoenholz2017correspondence,Karakida2018,hayou2018selection,hanin2018start, yang2018deep}.  These studies revealed a maximum depth through which signals can propagate at initialization, and verified empirically that networks are trainable precisely when signals can travel all the way through them. In the fully-connected setting, the theory additionally predicts the existence of an order-to-chaos phase transition in the space of initialization hyperparameters. For networks initialized on the critical line separating these phases, signals can propagate indefinitely and arbitrarily deep networks can be trained. While mean field theory captures the ``average'' dynamics of random neural networks it does not quantify the scale of gradient fluctuations that are crucial to the stability of gradient descent. A related body of work~\citep{saxe2013exact,pennington2017,PenningtonSG18} has examined the input-output Jacobian and used random matrix theory to quantify the distribution of its singular values in terms of the activation function and the distribution from which the initial random weight matrices are drawn. These works concluded that networks can be trained most efficiently when the Jacobian is well-conditioned, a criterion that can be achieved with orthogonal, but not Gaussian, weight matrices. Together, these approaches have allowed researchers to efficiently train extremely deep network architectures, but so far they have been limited to neural networks composed of fully-connected layers. 

In the present work, we continue this line of research and extend it to the convolutional setting.
We show that a well-defined mean-field theory exists for convolutional networks in the limit that the number of channels is large, even when the size of the image is small. Moreover, convolutional networks have precisely the same order-to-chaos transition as fully-connected networks, with vanishing gradients in the ordered phase and exploding gradients in the chaotic phase. And just like fully-connected networks, very deep CNNs that are initialized on the critical line separating those two phases can be trained with relative ease.

Moving beyond mean field theory, we additionally show that the random matrix analysis of~\cite{pennington2017,PenningtonSG18} carries over to the convolutional setting. Furthermore, we identify an efficient construction from the wavelet literature that generates random orthogonal matrices with the block-circulant structure that corresponds to convolution operators. This construction facilitates random orthogonal initialization for convolulational layers and enables good conditioning of the end-to-end Jacobian matrices of arbitrarily deep networks. We show empirically that networks with this initialization can train significantly more quickly than standard convolutional networks.

Finally, we emphasize that although the order-to-chaos phase boundaries of fully-connected and convolutional networks look identical, the underlying mean-field theories are in fact quite different. In particular, a novel aspect of the convolutional theory is the existence of multiple depth scales that control signal propagation at different spatial frequencies. In the large depth limit, signals can only propagate along modes with minimal spatial structure; all other modes end up deteriorating, even at criticality. We hypothesize that this type of signal degradation is harmful for generalization, and we develop a modified initialization scheme that allows for balanced propagation of signals among all frequencies. In this scheme, which we call Delta-Orthogonal initialization, the orthogonal kernel is drawn from a spatially non-uniform distribution, and it allows us to train vanilla CNNs of 10,000 layers or more with no degradation in performance.

\section{Theoretical results}\label{sec:theory}
In this section, we first derive a mean field theory for signal propagation in random convolutional neural networks. We will follow the general methodology established in~\citet{poole2016,schoenholz2016,yang2017}. We will then arrive at a theory for the singular value distribution of the Jacobian following~\citet{pennington2017, PenningtonSG18}. Together, this will allow us to derive theoretically motivated initialization schemes for convolutional neural networks that we call orthogonal kernels and Delta-Orthogonal kernels\footnote{An example implementation of a deep $\tanh$ network initialized critically using the Delta-Orthogonal kernel is provided at \url{https://github.com/brain-research/mean-field-cnns}.}. Later we will demonstrate experimentally that these kernels outperform existing initialization schemes for very deep vanilla convolutional networks.

\subsection{A mean field theory for CNNs}
\label{theory:meanfield0}
\subsubsection{Recursion relation for covariance}
\label{theory:meanfield}
    Consider an $L$-layer 1D\footnote{For notational simplicity, we consider one-dimensional convolutions, but the $d$-dimensional case proceeds identically.} CNN with periodic boundary conditions, filter width $2k+1$, number of channels $c$, spatial size $n$, per-layer weight tensors $\omega^l\in \mathbb{R}^{(2k+1) \times c \times c}$, and biases $b^l \in \mathbb{R}^{c}$. Let $\phi:\mathbb R \to \mathbb R$ be the activation function 
    and let $h^l_j(\alpha)$ denote the pre-activation unit at layer $l$, channel $j$, and spatial location $\alpha \in \sp$, where we define the set of spatial locations $\sp = \{1, ..., n\}$. The forward-propagation dynamics can be described by the recurrence relation,
    \begin{align}
    \label{eqn:hiter}
    h^{l+1}_j(\alpha) &= \sum_{\substack{i\in\chn \\ \beta \in \kn}} \phi(h^{l}_i(\alpha+\beta)) \omega_{ij}^{l+1}(\beta)+b_j^{l+1},
    \end{align}
    where $\kn =\{\beta\in\mathbb Z: |\beta|\leq k\}$ and
    $\chn=\{1, \dots, c\}$. At initialization, we take the weights $\omega_{ij}^{l}\left(\beta\right)$ to be drawn i.i.d. from the Gaussian $\mathcal N(0, \sigma_\omega^2/(c(2k+1)))$ and the biases $b^{l}_j$ to be drawn i.i.d. from the Gaussian $ \mathcal N(0, \sigma_b^2)$. 
    Note that $h^{l}_i(\alpha) = h^{l}_i(\alpha+n) =h^{l}_i(\alpha-n)$ since we assume periodic boundary conditions. We wish to understand how signals propagate through these networks. As in previous work in this vein, we will take the large network limit, which in this context corresponds to the number of channels $c\to\infty$. This allows us to use powerful theoretical tools such as mean field theory and random matrix theory. Moreover, this approximation has been shown to give results that agree well with experiments on finite-size networks.

    In the limit of a large number of channels, the central limit theorem implies that the pre-activation vectors $h_j^l$ are i.i.d. Gaussian with mean zero and covariance matrix $\Sigma^l_{\alpha,\alpha'} = \mathbb E[h^l_j(\alpha)h^l_j(\alpha')]$. Here, the expectation is taken over the weights and biases and it is independent of the channel index $j$.
    In this limit, the covariance matrix takes the form (see Supplemental Materials (SM)), 
    \begin{equation}
    \Sigma^{l+1}_{\alpha, \alpha'} 
    =\sigma_b^2 + \frac{\sigma_w^2} {2k\!+\!1}\! \sum_{\beta\in {\it ker}} \mathbb E\!\left[ \phi(h_j^l(\alpha\!+\!\beta))\phi(h_j^l(\alpha'\!+\!\beta)) \right],
    \end{equation}
    and is independent of $j$. A more compact representation of this equation can be given as,
    \begin{equation}
    \label{eqn:AC-map}
        \bm \Sigma^{l+1} \equiv \mathcal A\star \mathcal C (\bm\Sigma^{l})\,,
    \end{equation}
    where $\mathcal{A} = \frac{1}{2k+1}{\bf I}_{2k+1}$ and $\star$ denotes 2D circular cross-correlation, i.e. for any matrix ${\bm C}$, $\mathcal{A}\star{\bm C}$ is defined as,
    \begin{equation}
    \label{eqn:A-map}
    [\mathcal A\star {\bf C}]_{\alpha, \alpha'} = \frac{1}{2k+1}\sum_{\beta\in\kn} C_{\alpha+\beta, \alpha'+\beta}\,.
    \end{equation}
     The function $\mathcal C:\text{PSD}_n\to\text{PSD}_n$ is related to the $\mathcal C$-map defined in~\citet{poole2016} (see also \cite{daniely2016}) and is given by,
    \begin{equation}\label{eq:cmap}
    [\mathcal C(\bm \Sigma)]_{\alpha, \alpha'} = \sigma_\omega^2\,\mathbb E_{\bm h \sim\mathcal N(0,\bm \Sigma)} \left[ \phi(h_\alpha)\phi(h_{\alpha'}) \right] +\sigma_b^2.
    \end{equation}
    All but the two dimensions $\alpha$ and $\alpha'$ in eqn.~\eqref{eq:cmap} marginalize, so, as in~\cite{poole2016}, the $\mathcal{C}$-map can be computed by a two-dimensional integral. Unlike in~\citep{poole2016}, $\alpha$ and $\alpha'$ do not correspond to different examples but rather to different spatial positions and eqn.~\eqref{eq:cmap} characterizes how signals from a single input propagate through convolutional networks in the mean-field approximation\footnote{The multi-input analysis proceeds in precisely the same manner as we present here, but comes with increased notational complexity and features no qualitatively different behavior, so we focus our presentation on the single-input case.}.
    
    \subsubsection{Dynamics of signal propagation}
    We now seek to study the dynamics induced by eqn.~\eqref{eqn:AC-map}. Schematically, our approach will be to identify fixed points of eqn.~\eqref{eqn:AC-map} and then linearize the dynamics around these fixed points. These linearized dynamics will dictate the stability and rate of decay towards the fixed points, which determines the depth scales over which signals in the network can propagate.
    
    \citet{schoenholz2016} found that for many activation functions $\phi$ (e.g. $\tanh$) and any choice of $\sigma_w$ and $\sigma_b$, the ${\mathcal C}$-map has a fixed point $\bm \Sigma^*$ (i.e. $\mathcal C(\bm \Sigma^*) = \bm \Sigma^*$) of the form, 
    \begin{equation}
    \label{eqn:fixed-point-equation}
        {\Sigma}^*_{\alpha,\alpha'} = q^*(\delta_{\alpha,\alpha'} + (1-\delta_{\alpha, \alpha'})c^*)\,,
    \end{equation} 
    where $\delta_{a,b}$ is the Kronecker-$\delta$, $q^*$ is the fixed-point variance of a single input, and $c^*$ is the fixed-point correlation between two inputs. It follows from the form of eqn.~(\ref{eqn:A-map}) that ${\bm \Sigma}^*$ is also a fixed point of the layer-to-layer covariance map in the convolutional case (eqn.~(\ref{eqn:AC-map})), i.e. ${\bm \Sigma}^* = \mathcal{A}\star \mathcal{C}({\bm \Sigma^*})$.

    To analyze the dynamics of the iteration map~(\ref{eqn:AC-map}) near the fixed point $\bm \Sigma^*$, we define $\bm \epsilon^l = \bm \Sigma^* - \bm \Sigma^l$ and expand eqn.~\eqref{eqn:AC-map} to lowest order in $\bm \epsilon.$ This expansion requires the Jacobian of the $\mathcal{C}$-map evaluated at the fixed point, the properties of which we analyze in the SM. In brief, perturbations in $q^*$ and $c^*$ evolve independently and the Jacobian decomposes into a diagonal eigenspace $V_{\text{d}}$ with eigenvalue $\chi_{q^*}$, and an off-diagonal eigenspace $V_{\text{o.d.}}$ with eigenvalue $\chi_{c^*}$. The eigenvalues are given by\footnote{By the symmetry of $\bm \Sigma^*$, these expectations are independent of spatial location and of the choice of $h_1$ and $h_2$.},
    \begin{equation}
    \begin{split}
    \label{eqn:chi_c}
        \chi_{c^*} &= \sigma_w^2\mathbb E_{\bm h\sim\mathcal N(0,\bm C^*)}[\phi'(h_1)\phi'(h_2)]\,,\; h_1\neq h_2\,,\\
        \chi_{q^*} &= \sigma_w^2\mathbb E_{\bm h\sim\mathcal N(0,\bm C^*)}[\phi''(h_1)\phi(h_1)+\phi'(h_1)^2]\,,
    \end{split}
    \end{equation}
    and the eigenspaces have bases,
    \begin{equation}
    \begin{split}
        \label{eqn:decomposition}
        B_{\text{d}} &= \{M^{\alpha,\alpha}: M^{\alpha,\alpha}_{\bar\alpha, \bar\alpha'} = \gamma \delta_{\alpha, \bar\alpha}\delta_{\alpha, \bar\alpha'} +\delta_{\bar\alpha, \alpha} + \delta_{\bar\alpha', \alpha}\}_{\scriptscriptstyle{\alpha\in\sp}}\\
        B_{\text{o.d.}} &= \{M^{\alpha, \alpha'}: M^{\alpha, \alpha'}_{{\bar\alpha, \bar\alpha'}} = \delta_{\alpha, \bar\alpha}\delta_{\alpha', \bar\alpha'}+\delta_{\alpha, \bar\alpha'}\delta_{\alpha', \bar\alpha}\}_{\alpha \neq\alpha'}\,,
    \end{split}
    \end{equation}
    i.e. $V_\text{d} = \spn(B_\text{d})$ and $V_{\text{o.d}} = \spn(B_{\text{o.d.}})$. Note that $\chi_{q^*}$ and $\chi_{c^*}$ also were found in~\citet{schoenholz2016} to control signal propagation in the fully-connected case. The constant $\gamma$ is given in Lemma \ref{lemma:representation} of the SM but does not concern us here. This eigen-decomposition implies that the layer-wise deviations from the fixed point evolve under eqn.~(\ref{eqn:AC-map}) as,
    \begin{equation}\label{eqn:dyn_linearized}
    \bm \epsilon^{l+1} = \chi_{q^*}\mathcal A\star{\bm \epsilon}_{\text{d}}^l+\chi_{c^*}\mathcal A\star{\bm \epsilon}_{\text{o.d.}}^l + \mathcal O((\bm\epsilon^l)^2)\,,
    \end{equation}
    where ${\bm \epsilon}_{\text{d}}$ and ${\bm \epsilon}_{\text{o.d.}}$ are decomposition of ${\bm \epsilon}$ into the eigenspaces $V_{\text{d}}$ and $V_{\text{o.d.}}$.
    
    Eqn.~(\ref{eqn:dyn_linearized}) defines the linear dynamics of random convolutional neural networks near their fixed points and is the basis for the in-depth analysis of the following subsections.
    
    \subsubsection{Multi-dimensional signal propagation}
    \label{sec:multi}
    In the fully-connected setting, the dynamics of signal propagation near the fixed point are governed by scalar evolution equations. In contrast, the convolutional setting enjoys much richer dynamics, as eqn.~(\ref{eqn:dyn_linearized}) describes a multi-dimensional system that we now analyze.
    
    It follows from eqns.~(\ref{eqn:A-map}) and~(\ref{eqn:decomposition}) (see also the SM) that $\mathcal A$ does not mix the diagonal and off-diagonal eigenspaces, i.e. $\mathcal{A}\star{\bm \epsilon}_{\text{d}} \in V_\text{d}$ and $\mathcal{A}\star{\bm \epsilon}_{\text{o.d.}} \in V_\text{o.d.}$. To see this, note that for $M^{\alpha,\alpha'}\in V_{\text{o.d.}}$, the definition implies $M^{\alpha,\alpha'}_{\bar\alpha + \beta,\bar\alpha'+\beta} = M^{\alpha - \beta,\alpha' - \beta}_{\bar\alpha,\bar\alpha'}$. This property ensures that $\mathcal A\star M^{\alpha,\alpha'}$ can be expressed as a linear combination of matrices in $V_{\text{o.d.}}$, which means it also belongs to $V_{\text{o.d}}$. The same argument applies to $M^{\alpha,\alpha}\in V_{\text {d.}}$. As a result, these eigenspaces evolve entirely independently under the linearization of the covariance iteration map~(\ref{eqn:AC-map}).
    
    Let $l_0$ denote the depth over which transient effects persist and after which eqn.~(\ref{eqn:dyn_linearized}) accurately describes the linearized dynamics. Therefore, at depths larger than $l_0$, we have
    \begin{equation}\label{eqn:dyn_linearized_power_qc}
        \bm \epsilon^{l} \approx \underbrace{\mathcal A\star\cdots\mathcal A\,\star}_{l - l_0}\,(\chi_{q^*}^{l-l_0}\bm\epsilon_{\text{d}}^{l_0} + \chi_{c^*}^{l-l_0}\bm\epsilon_{\text{o.d.}}^{l_0})\,.
    \end{equation}
    This matrix-valued equation is still somewhat complicated owing to the nested applications of $\mathcal A$. To further elucidate the dynamics, we can move to a Fourier basis, which diagonalizes the circular cross-correlation operator and decouples the modes of eqn.~(\ref{eqn:dyn_linearized_power_qc}). In particular, let $\mathcal{F}$ denote the 2D discrete Fourier transform and $\tilde{\epsilon}_{\alpha,\alpha'} \equiv \mathcal{F}({\bm \epsilon})_{\alpha,\alpha'}$ denote a Fourier mode of $\bm \epsilon$. Then eqn.~(\ref{eqn:dyn_linearized_power_qc}) becomes a simple scalar equation,
    \begin{equation}
    \label{eqn:dyn_linearized_freq}
        \tilde{\epsilon}_{\alpha,\alpha'}^l \approx (\lambda_{\alpha,\alpha'}\chi_{q^*})^{l-l_0}[\tilde{\epsilon}_{\text{d}}^{l_0}]_{\alpha,\alpha'} + (\lambda_{\alpha,\alpha'}\chi_{c^*})^{l-l_0}[\tilde{\epsilon}_{\text{o.d.}}^{l_0}]_{\alpha,\alpha'}\,,
    \end{equation}
    with $\lambda_{\alpha,\alpha'} = \mathcal{F}(\mathcal{A})^*_{\alpha,\alpha'}$. Thus, the linearized dynamics of convolutional neural networks decouple into independently-evolving Fourier modes that evolve near the fixed point at frequency-dependent rates.

    \subsubsection{Fixed-point analysis}
    
    The stability of the fixed point $\Sigma^*$ is determined by whether nearby points move closer or farther from $\Sigma^*$ under the dynamics described by eqn.~(\ref{eqn:dyn_linearized}). Eqn.~(\ref{eqn:dyn_linearized_freq}) shows that this condition depends on the whether the quantities $\lambda_{\alpha,\alpha'}\chi_{q^*}$ and $\lambda_{\alpha,\alpha'}\chi_{c^*}$ are less than or greater than one.
    
    Since $\mathcal A$ is a diagonal matrix, the eigenvalues $\lambda_{\alpha,\alpha'}$ have a specific structure. In particular, the set of eigenvalues is comprised of $n$ copies of the 1D discrete Fourier transform of the diagonal entries of $\mathcal A$. Furthermore, since the diagonal entries of $\mathcal A$ are non-negative and sum to one, their Fourier coefficients have absolute value no larger than one and the zero-frequency coefficient is equal to one; see Figure \ref{fig:depth} for the full distribution in the case of 2D convolutions. It follows that the fixed point $\Sigma^*$ will be stable if and only if $\chi_{q^*} < 1$ and $\chi_{c^*} < 1$.
    
    These stability conditions are precisely the ones found to govern  fully-connected networks~\cite{poole2016, schoenholz2016}. Moreover, the fixed point matrix $\Sigma^*$ is also the same as in the fully-connected case. Together, these observations imply that the entire fixed-point structure of the convolutional case is identical to that of the fully-connected case. In particular, based on the results of~\cite{poole2016}, we can immediately conclude that the $(\sigma_w, \sigma_b)$ hyperparameter plane is separated by the line $\chi_1 = 1$ into an ordered phase with $c^* = 1$ in which all pixels approach the same value, and a chaotic phase with $c^* < 1$ in which the pixels become decorrelated with one another; see the SM for a review of this phase diagram analysis.

    \begin{figure}[t]
    \vspace{0.1cm}
     \centering
     \includegraphics[width=0.9\columnwidth]{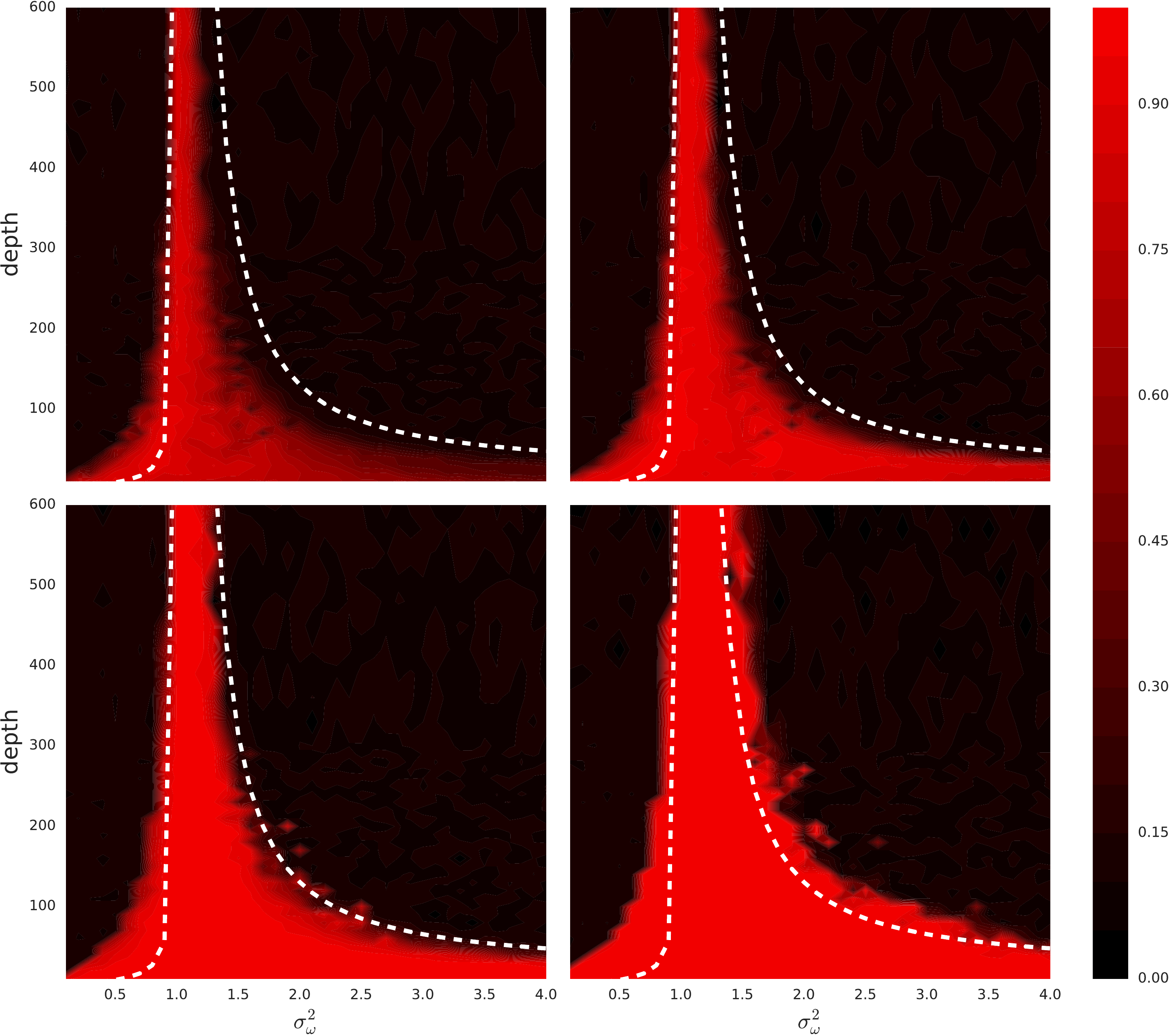}
      \caption{Mean field theory predicts the maximum trainable depth for CNNs. 
     For fixed bias variance $\sigma^2_b =2\times 10^{-5}$, the heat map shows the training accuracy on MNIST obtained for a given depth $L$ network and weight variance $\sigma_w$, after (a) $500$, (b) $2{,}500$, (c) $10{,}000$, and (d) $100{,}000$ training steps. Also plotted (white dashed line) is a multiple ($6\xi_c$) of the characteristic depth scale governing convergence to the fixed point.
     }
     \vspace{-0.3cm}
     \label{fig:heatmaps}
    \end{figure}

    \subsubsection{Depth scales of signal propagation}
    \label{sec:depth}
    We now assume that the conditions for a stable fixed point are met, i.e. $\chi_{q^*} < 1$ and $\chi_{c^*} < 1$, and we consider the rate at which the fixed point is approached. As in~\cite{schoenholz2016}, it is convenient to additionally assume $\chi_{q^*} < \chi_{c^*}$ so that the dynamics in the diagonal subspace can be neglected. In this case, eqn.~\eqref{eqn:dyn_linearized_freq} can be rewritten as
    \begin{equation}
        \label{eqn:depth_scale}
        \tilde{\epsilon}_{\alpha,\alpha'}^l \approx  e^{-(l-l_0)/\xi_{\alpha,\alpha'}} [\tilde{\epsilon}_{\text{o.d.}}]_{\alpha,\alpha'}^{l_0}\,,
    \end{equation}
    where $\xi_{\alpha,\alpha'} = -1/\log(\chi_{c^*}\lambda_{\alpha,\alpha'})$ are depth scales governing the convergence of the different modes. In particular, we expect signals corresponding to a specific Fourier mode $f_{\alpha,\alpha'}$ to be able to travel a depth commensurate to $\xi_{\alpha,\alpha'}$ through the network. Thus, unlike fully-connected networks which exhibit only a single depth scale, convolutional networks feature a hierarchy of depth scales.
    
    Recalling that $\lambda_{\alpha,n-\alpha} = 1$, it follows that $\xi_c\equiv \xi_{\alpha,n-\alpha} = -1/\log\chi_{c^*}$, which is identical to the depth scale governing signal propagation through fully-connected networks. It follows from~\cite{schoenholz2016} that when $\chi_1 = 1$, $\xi_{\alpha,n-\alpha}$ diverges and thus convolutional networks can propagate signals arbitrarily far through the $f_{\alpha,n-\alpha}$ modes. Since $|\lambda_{\alpha,\alpha'}| < 1$ for $\alpha' \neq n-\alpha$, these are the only modes through which signals can propagate without attenuation. 
    Finally, we note that the $f_{\alpha,n-\alpha}$ modes correspond to perturbations that are spatially uniform along the cyclic diagonals of the covariance matrix.
    The fact that all signals with additional spatial structure attenuate for large depth suggests that deep critical convolutional networks behave quite similarly to fully-connected networks, which also cannot propagate spatially-structured signals.

    \begin{figure}[t]
    \begin{center}
    \vspace{0.2cm}
    \centerline{\includegraphics[width=0.7\columnwidth]{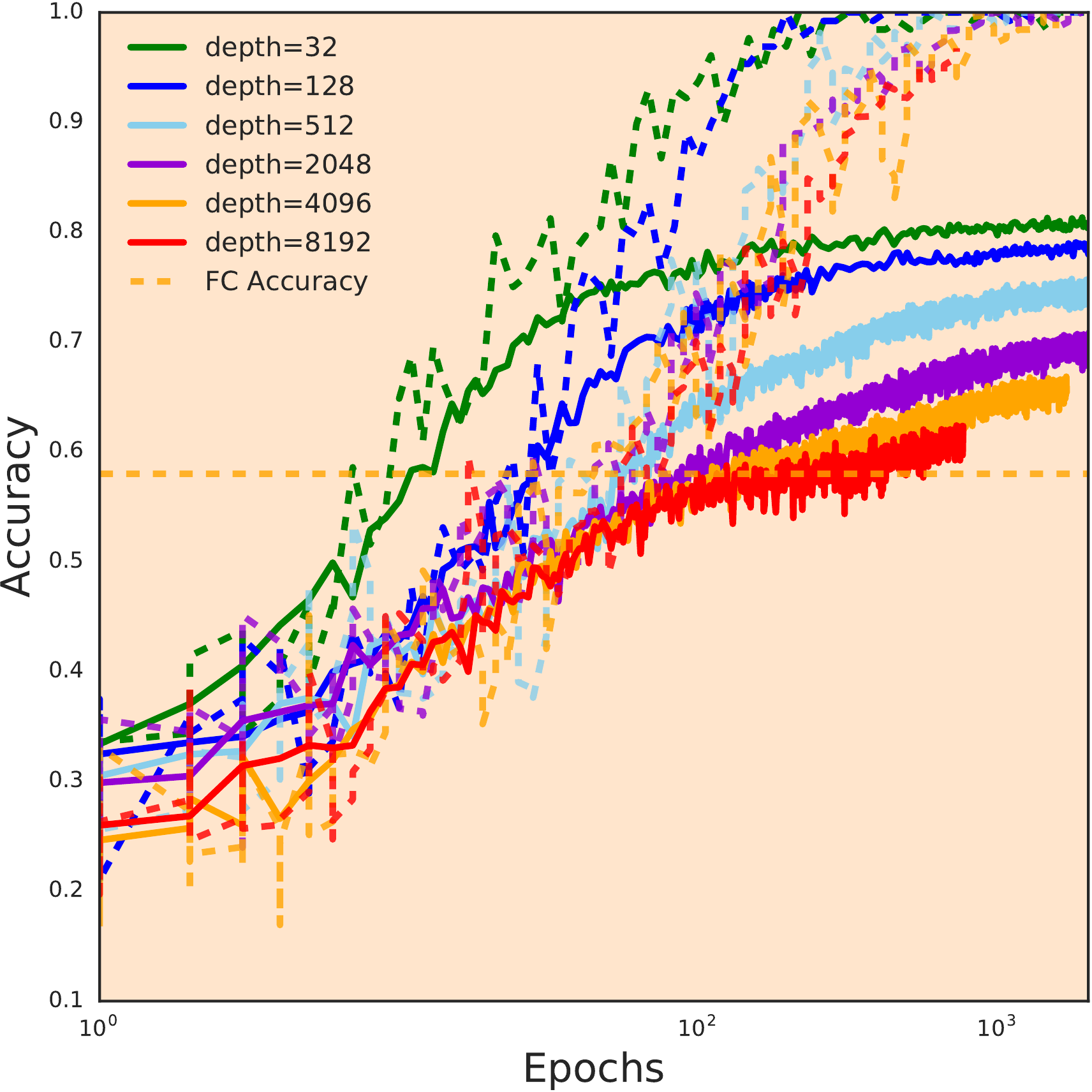}}
    \caption{
    Test (solid) and training (dashed) curves of CNNs with different depths initialized critically using orthogonal kernels on CIFAR-10. Training accuracy reaches $100\%$ for all these curves (except for 8192, which was stopped early) but generalization performance degrades with increasing depth, likely because of attenuation of spatially non-uniform modes. The Delta-Orthogonal initialization in Fig.~\ref{fig:cm10000} addresses this reduction in test performance with increasing depth. 
    }
    \label{fig:cifar10000}
    \end{center}
    \vspace{-0.75cm}
    \end{figure}

    \subsubsection{Non-uniform kernels}
    \label{sec:non-uniform}

    \begin{figure*}[t]
     \vspace{0.1cm}
     \includegraphics[width=.98\textwidth]{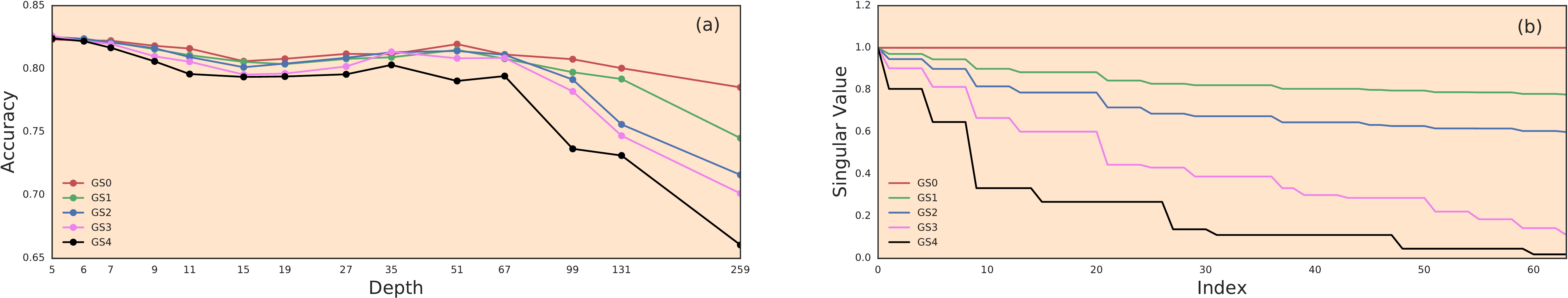}
     \vspace{-0.1cm}
     \caption{
     Test performance, as a function of depth, is correlated with the singular value distribution (SVD) of the generalized averaging operator $(\mathcal{A}_v\star)$ (see eqn.~(\ref{eqn:ACv-map})).
     (a) Initialized critically, we examine the test accuracy of CNNs with different depths and with Gaussian initialization of different non-uniform variance vectors. We ``deform" the variance vector from a delta function (red) to a uniformly distributed one (black). Starting from depth 35, we see the test accuracy curve also ``deforms'' from the red one to the black one. (b) The SVD of $(\mathcal{A}_v\star)$ for the selected variance vectors. The $x$-axis represents the index of a singular value, with a total of 64 singular values (each has 64 copies) for each variance vector. See Section~\ref{exp:multi} for details.
     }
     \vspace{-0.1in}
     \label{fig:depth}
    \end{figure*}
    
    The similarities between signal propagation in convolutional neural networks and fully-connected networks in the limit of large depth are surprising. A consequence may be that the performance of very deep convolutional networks degrades as the signal is forced to propagate along modes with minimal spatial structure. Indeed, Fig.~\ref{fig:cifar10000} shows that the generalization performance decreases with depth, and that for very large depth it barely surpasses the performance of a fully-connected network.
    
    If increased spatial uniformity is the problem, eqn.~\eqref{eqn:depth_scale} holds the solution. In order for \emph{all} modes to propagate without attenuation, it is necessary that $\lambda_{\alpha,\alpha'} = 1$ for all $\alpha, \alpha'$. In fact, it is easy to show that the distribution of $\{\lambda_{\alpha,\alpha'}\}$ can be modified by allowing for spatial non-uniformity in the variance of the weights within the kernel. To this end, we introduce a non-negative vector $v=(v_\beta)_{\beta\in \kn}$ chosen such that $\sum_\beta v_\beta = 1$, and initialize the weights of the network according to $w^l_{ij}(\beta)\sim\mathcal N(0, \sigma_w^2v_\beta/c)$. Each choice of $v$ will induce a new dynamical equation analogous to eqn.~\eqref{eqn:AC-map} (see SM),
    \begin{equation}\label{eqn:ACv-map}
        \bm\Sigma^{l+1} = \mathcal A_v \star \mathcal C (\bm\Sigma^l)\,, 
    \end{equation}
    where $\mathcal{A}_v = \diag(v).$
    It follows directly from the previous analysis that the linearized dynamics of eqn.~\eqref{eqn:ACv-map} will be identical to the dynamics of eqn~\eqref{eqn:AC-map}, only now with $\lambda_{\alpha,\alpha'} = \mathcal{F}(\mathcal{A}_v)_{\alpha,\alpha'}^*$. By the same argument presented in Section~\ref{sec:multi}, the set of eigenvalues is now comprised of $n$ copies of the 1D Fourier transform of $v$. As a result, it is possible to control the depth scales over which different modes of the signal can propagate through the network by changing the variance vector $v$. We will return to this point in section~\ref{theory:orthogonal2}.

\subsection{Back-propagation of signal}
We now turn our attention to the back-propgation of error signals through a convolutional network. Let $E$ denote the loss and $\delta^l_j(\alpha)$ the back-propagated signal at layer $l$, channel $j$ and spatial location $\alpha$, i.e., 
    \begin{equation}
    \delta^l_j(\alpha) = \frac {\partial E}{ \partial {h^l_j(\alpha)}}. 
    \end{equation}
    The recurrence relation is given by
    \begin{align*}
    \delta^l_j(\alpha)
    &= \sum_{i\in {\it chn}}\sum_{\beta\in{\it ker}}  \delta^{l+1}_{i}(\alpha -\beta)\omega_{ji}^{l+1}(\beta)\phi'(h_j^l(\alpha)). 
    \end{align*}
    As in~\cite{schoenholz2016}, we additionally make the assumption that the weights used during back-propagation are drawn independently from the weights used in forward propagation, in which case the random variables $\{\delta^l_j\}_{j\in \it chn}$ are independent for each $l$. The covariance matrices $\tilde {\bm \Sigma}^l \equiv \mathbb E\left[ \delta^{l}_j (\delta^{l}_j)^T \right]$ back-propagate according to,
    \begin{equation}
    \tilde {\bm \Sigma}^l_{\alpha,\alpha'} \!=\! 
    \sum_{\beta\in{\kn}} v_\beta
    \tilde {\bm \Sigma}^{l+1}_{\alpha -\beta, \alpha' - \beta} \cdot
    \sigma_w^2\mathbb E_{\bm h\sim\mathcal N(0,\bm \Sigma^l)}[\phi'(h_\alpha)\phi'(h_{\alpha'})]\,.\\
    \end{equation}
    We are primarily interested in the diagonal of $\tilde {\bm \Sigma}^l$, which measures the variance of back-propagated signals. We will also assume $l>l_0$ (see section~\ref{sec:multi}) so that ${\bm \Sigma}^l$ is well-approximated by ${\bm \Sigma}^*$. In this case,
    \begin{equation}
        \tilde {\bm \Sigma}^l_{\alpha,\alpha} \approx \chi_{1} \sum_{\beta\in{\kn}} v_\beta \tilde {\bm \Sigma}^{l+1}_{\alpha -\beta, \alpha - \beta}\,,
    \end{equation}
    where we used eqn.~(\ref{eqn:chi_c}). Therefore we find that, $\tilde {\bm \Sigma}^l_{\alpha,\alpha} \sim \chi_1^{L-l}\tilde {\bm \Sigma}^L_{\alpha,\alpha} $, where $L$ is the total depth of the network. As in the fully-connected case, $\chi_1 = 1$ is a necessary condition for gradient signals to neither explode nor vanish as they back-propagate through a convolutional network. However, as discussed in~\cite{pennington2017, PenningtonSG18}, this is not always a sufficient condition for trainability. To further understand backward signal propagation, we need to push our analysis beyond mean field theory.

\subsubsection{Beyond mean field theory}
\label{sec:beyond}
We have observed that the quantity $\chi_1$ is crucial for determining signal propagation in CNNs, both in the forward and backward directions. As discussed in~\cite{poole2016}, $\chi_1$ equals the  the mean squared singular value of the Jacobian $\bm J^l$ of the layer-to-layer transition operator. 
Beyond just the second moment, higher moments and indeed the whole distribution of singular values of the entire end-to-end Jacobian ${\bm J} = \prod_{l} {\bm J}^l$ are important for ensuring trainability of very deep fully-connected networks~\cite{pennington2017, PenningtonSG18}. Specifically, networks train well when their input-output Jacobians exhibit \emph{dynamical isometry}, namely the property that the entire distribution of singular values is close to $1$.

In fact, we can adopt the entire analysis of~\cite{pennington2017, PenningtonSG18} into the convolutional setting with essentially no modification. The reason stems from the fact that, because convolution is a linear operator, it has a matrix representation, $\bm W^l$, which appears in the end-to-end Jacobian in precisely the same manner as do the weight matrices in the fully-connected case. In particular, $\bm J = \prod_{l=1}^L {\bm D}^l {\bm W}^l$,
where ${\bm D^l}$ is the diagonal matrix whose diagonal elements contain the vectorized representation of derivatives of post-activation neurons in layer $l$. Roughly speaking, since this is the same expression as in~\cite{pennington2017, PenningtonSG18}, the conclusions found in that work regarding dynamical isometry apply equally well in the convolutional setting.

The analysis of \citet{pennington2017, PenningtonSG18} reveals that the singular values of $\bm J$ depends crucially on the distribution of singular values of $\bm W^l$ and $\bm D^l$. In particular, to achieve dynamical isometry, all of these matrices should be close to orthogonal. As in the fully-connected case, the singular values of $\bm D^l$ can be made arbitrarily close to $1$ by choosing a small value for $q^*$ and by using an activation function like $\tanh$ that is smooth and linear near the origin. In the convolutional setting, the matrix representation of the convolution operator $\bm W^l$ is a $c\times c$ block matrix with $n\times n$ circulant blocks. Note that in the large $c$ limit, $n/c \to 0$ and the relative size of the blocks vanishes. Therefore, if the weights are i.i.d. random variables, we can invoke universality results from random matrix theory to conclude its singular value distribution converges to the Marcenko-Pastur distribution; see Fig.~\ref{fig:SVD_weights} in the SM. As such, we find that CNNs with i.i.d. weights cannot achieve dynamical isometry. We address this issue in the next section.

\subsection{Orthogonal Initialization for CNNs}
\label{theory:orthogonal1}
In \cite{pennington2017, PenningtonSG18}, it was observed that dynamical isometry can lead to dramatic improvements in training speed, and that achieving these favorable conditions requires orthogonal weight initializations. While the procedure to generate random orthogonal weight matrices in the fully-connected setting is well-known, it is less obvious how to do so in the convolutional setting, and at first sight it is not at all clear whether it is even possible. We resolve this question by invoking a result from the wavelet literature~\cite{kautsky1994} and provide an explicit construction. We will focus on the {\it two-dimensional} convolution here and begin with some notation.  

\label{theory:orthogonal}
 \begin{figure}[t]
    \begin{center}
    \centerline{\includegraphics[width=0.6\columnwidth]{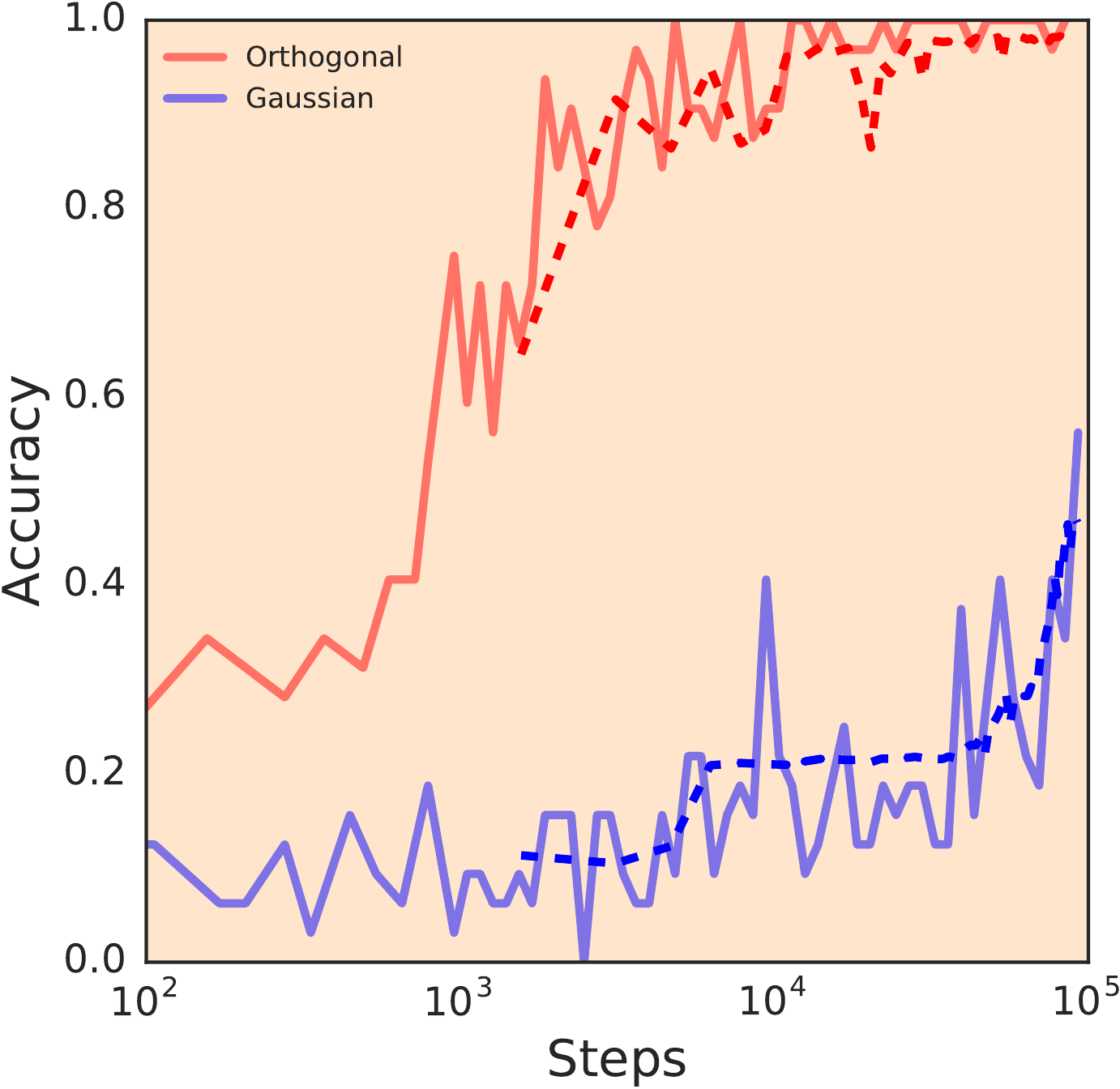}}
    \caption{Orthogonal initialization leads to faster training in CNNs. Training (solid lines) and test curves for a 4,000-layer CNN trained using orthogonal (red) and Gaussian (blue) initializations with identical weight variance. }
    \label{fig:d4000}
    \end{center}
    \vskip -0.2in
\end{figure}

\begin{definition}We say $K\in \mathbb R^{\Bbbk\times \Bbbk\times c_{in} \times c_{out}}$ is an orthogonal kernel if for all $x\in \mathbb R^{n\times n \times c_{in}}$, $\|K\ast x\|_2 = \|x\|_2$.
\end{definition}
    \begin{definition}
    Consider the block matrices $B=\{B_{i,j}\}_{0\leq i, j\leq p-1} \in \mathbb{R}^{pn\times pn}$ and $C=\{C_{i,j}\}_{0\leq i, j\leq q-1} \in \mathbb{R}^{qn\times qn}$, with constituent blocks $B_{i,j}\in\mathbb{R}^{n\times n}$ and $C_{i,j}\in\mathbb{R}^{n\times n}$. Define the block-wise convolution operator $\Box$ by,
    \begin{equation}
    \label{eqn:box}
    \left[B \Box C\right]_{i,j} = \sum_{i',j'} B_{i',j'} C_{i-i', j-j'},
    \end{equation}
    where the out-of-range matrices are taken to be zero.
    \end{definition}
 
    Algorithm~\ref{alg:orthogonal} shows how to construct orthogonal kernels for 2D convolutions of size $\Bbbk \times \Bbbk \times c_{in} \times c_{out}$ with $c_{in} \leq c_{out}$. One can employ the same method to construct kernels of higher (or lower) dimensions. This new initialization method can dramatically boost the learning speed of deep CNNs; see Fig.~\ref{fig:d4000} and Section~\ref{exp:orthogonal1}.

\subsection{Delta-Orthogonal Initialization} 
\label{theory:orthogonal2}

    In Section~\ref{sec:depth} it was observed that, in contrast to fully-connected networks, CNNs have multiple depth scales controlling propagation of signals along different Fourier modes. Even at criticality, for generic variance-averaging vectors $v$, the majority of these depth scales are finite. However, there does exist one special averaging vector for which all of the depth scales are infinite: a one-hot vector, i.e. $v_i = \delta_{k,i}$. 
    This kernel places all of its variance in the spatial center of the kernel and zero variance elsewhere. In this case, the eigenvalues $\lambda_{\alpha,\alpha'}$ are all equal to $1$ and all depth scales diverge, implying that signals can propagate arbitrarily far along all Fourier modes. 
    
    If we combine this special averaging vector with the orthogonal initialization of the previous section, we obtain a powerful new initialization scheme that we call Delta-Orthogonal Initialization. Matrices of this type can be generated from Algorithm~\ref{alg:orthogonal} with $\Bbbk=1$ and padding with appropriate zeros or directly from Algorithm~\ref{alg:delta} in the SM.
    
    In the following sections, we demonstrate experimentally that extraordinarily deep convolutional networks can be trained with these initialization techniques.

\begin{algorithm}[t!]
    \caption{2D orthogonal kernels for CNNs, available in TensorFlow via the $\operatorname{ConvolutionOrthogonal}$ initializer.}
    \label{alg:orthogonal}
    
    \begin{algorithmic}
    \STATE {\bfseries Input:} $\Bbbk$ kernel size, $c_{in}$ number of input channels, $c_{out}$ number of output channels.
    
    {\bf Return:} a $\Bbbk\times \Bbbk \times c_{in}\times c_{out}$ tensor $K$. 
    
       {\bf Step 1.} Let $K$ be the $1\times 1 \times c_{out}\times c_{out}$ tensor such that $K[0, 0]= I$, where $I$ is the $c_{out}\times c_{out}$ identity matrix. 
       
       {\bf Step 2.} Repeat the following $(\Bbbk-1)$ times: 
       
        {Randomly generate two orthogonal projection matrices $P$ and $Q$ of size $c_{out}\times c_{out}$ and set (see eqn.~(\ref{eqn:box})})
                $$
                K \leftarrow K \Box       \begin{bmatrix}
        PQ       & P (1- Q) 
        \\
        (1-P) Q  & (1-P) (1-Q)
      \end{bmatrix}.  
                $$
            {\bf Step 3.} Randomly generate a $c_{in}\times c_{out}$ matrix $H$ with orthonormal rows and for $i = 0,\dots, \Bbbk-1 $ and $j = 0, \dots, \Bbbk-1$, set $K[i,j] \leftarrow HK[i,j]$. 
            
            {\bf Return} $K$. 
    \end{algorithmic}
    \end{algorithm}

\section{Experiments}
    
To support the theoretical results built up in Section \ref{sec:theory}, we trained a large number of very deep CNNs on MNIST and CIFAR-10 with $\tanh$ as the activation function. We use the following {\it vanilla} CNN architecture. First we apply three $3\times3\times c$ convolutions with strides 1, 2 and 2 in order to increase the channel size to $c$ and reduce the spatial dimension to $7\times7$ (or $8\times 8$ for CIFAR-10), and then a block of $d$ $3\times3\times c$ convolutions with $d$ varying from $2$ to $10,000$. Finally, an average pooling layer and a fully-connected layer are applied. Here $c=256$ when $d\leq 256$ and $c=128$ otherwise. To maximally support our theories, we applied {\it no} common techniques (including learning rate decay). Note that the early downsampling is necessary from a computational perspective, but it does diminish the maximum achievable performance; e.g. our best achieved test accuracy with downsampling was 82$\%$ on CIFAR-10. We performed an additional experiment training a 50 layers network without downsampling. This resulted in a test accuracy of $89.90\%$, which is comparable to the best performance on CIFAR-10 using a $\tanh$ architecture that we were able to find ($89.82\%$, \cite{Mishkin2015}).

\subsection{Trainability and Critical Initialization}
\label{exp:critical}

    The analysis in Section \ref{theory:meanfield0} gives a prediction for precisely which initialization hyperparameters a CNN will be trainable. In particular, we predict that the network ought to be trainable provided $L\lesssim \xi_c$. To test this, we train a large number of convolutional neural networks on MNIST with depth varying between $L = 10$ and $L = 600$ and with weights initialized with $\sigma_w^2 \in [0, 4]$. In Fig.~\ref{fig:heatmaps} we plot -- using a heatmap -- the training accuracy obtained by these networks after different numbers of steps. Additionally we overlay the depth scale predicted by our theory, $\xi_c$. We find strikingly good agreement between our theory of random networks and the results of our experiments. 

\subsection{Orthogonal Initialization and Ultra-deep CNNs}
\label{exp:orthogonal1}
    We argued in Section~\ref{sec:beyond} that the input-output Jacobian of CNNs with i.i.d. weights will become increasingly ill-conditioned as the number of layers grows. On the other hand, orthogonal weight initializations can achieve dynamical isometry and dramatically boost the training speed. To verify this, we train a 4,000-layer CNN on MNIST using a critically-tuned Gaussian weight initialization and the orthogonal initialization scheme developed in Section~ \ref{theory:orthogonal1}. Fig.~\ref{fig:d4000} shows that the network with Gaussian initialization learns slowly (test and training accuracy is below $60\%$ after $90,000$ steps, about 60 epochs). In contrast, orthogonal initialization learns quickly with test accuracy above $60\%$ after only 1 epoch, and achieves $95\%$ after $10,000$ steps or about 7 epochs. 

\subsection{Multi-dimensional Signal Propagation}
    \label{exp:multi}
    
    The analysis in Section \ref{sec:multi} and Section \ref{sec:non-uniform} suggest that CNNs initialized with kernels with spatially uniform variance may suffer a degradation in generalization performance as the depth increases. Fig.~\ref{fig:cifar10000} shows the learning curves of CNNs on CIFAR-10 with depth varying from $32$ to $8192$. Although the orthogonal initialization enables even the deepest model to reach $100\%$ training accuracy, the test accuracy decays as the depth increases with the deepest mode generalizing only marginally better than a fully-connected network.
    
    To test whether this degradation in performance may be the result of attenuation of spatially non-uniform signals, we trained a variety of models on CIFAR-10 whose kernels were initialized with spatially non-uniform variance. According to the analysis in Section~\ref{sec:non-uniform}, changing the shape of this non-uniformity controls the depth scales over which different Fourier components of the signal can propagate through the network. We examined five different non-uniform critical Gaussian initialization methods. The variance vectors $v$ were chosen in the following way: GS0 refers to the one-hot delta initialization for which the eigenvalues $\lambda_{\alpha,\alpha'}$ are all equal to 1. GS1, GS2 and GS3 are obtained by interpolating between GS0 and GS4, which is the uniform variance initialization.
    
    Each variance vector has exactly $8\times 8$ singular values, plotted in Fig.~\ref{fig:depth}(b) in descending order. Note that from GS0 to GS4, the singular values become more poorly-conditioned (the distribution becomes more concentrated around 0). Fig.~\ref{fig:depth}(a) shows that the relative fall-off of generalization performance with depth follows the same pattern: the more poorly-conditioned the singular values the worse the model generalizes.  These observations suggest that salient information may be propagating along multiple Fourier modes.

\subsection{Training 10,000-layers: Delta-Orthogonal Initialization.}
\label{exp:orthogonal2}
    Our theory predicts that an ultra-deep CNNs can train faster and perform better if
    critically initialized using Delta-Orthogonal kernels. To test this theory, we train CNNs of 1,250, 2,500, 5,000 and 10,000 layers on both MNIST and CIFAR-10 (Fig.~\ref{fig:cm10000}).
    All these networks learn surprisingly quickly and, 
    remarkably, the learning time measured in number
    of training epochs is independent of depth. 
    Furthermore, our experimental results match well with the predicted benefits of this initialization: $99\%$ test accuracy on MNIST for a 10,000-layer network, and $82\%$ on CIFAR-10. 
    To isolate the benefits of the Delta-Orthogonal init, we also train a 2048-layer CNN (Fig.~\ref{fig:cifar10000}) using the spatially-uniform orthogonal initialization proposed in Section~\ref{theory:orthogonal}; the testing accuracy is about  $70\%$. Note that the test accuracy using (spatially uniform) Gaussian (non-orthogonal) initialization is already below $70\%$ when the depth is 259.

\section{Discussion}
In this work, we developed a theoretical framework based on mean field theory to study the propagation of signals in deep convolutional neural networks. By examining the necessary conditions for signals to flow both forward and backward through the network without attenuation, we derived an initialization scheme that facilitates training of vanilla CNNs of unprecedented depths. We presented an algorithm for the generation of random orthogonal convolutional kernels, an ingredient that is necessary to enable dynamical isometry, i.e. good conditioning of the network's input-output Jacobian. In contrast to the fully-connected case, signal propagation in CNNs is intrinsically multi-dimensional -- we showed how to decompose those signals into independent Fourier modes and how to promote uniform signal propagation across them. By leveraging these various theoretical insights, we demonstrated empirically that it is possible to train vanilla CNNs with 10,000 layers or more.

Our results indicate that we have removed all the major fundamental obstacles to training arbitrarily deep vanilla convolutional networks. In doing so, we have layed the groundwork to begin addressing some outstanding questions in the deep learning community, such as whether depth alone can deliver enhanced generalization performance. Our initial results suggest that past a certain depth, on the order of tens or hundreds of layers, the test performance for vanilla convolutional architecture saturates. These observations suggest that architectural features such as residual connections and batch normalization are likely to play an important role in defining a good model class, rather than simply enabling efficient training.

\section*{Acknowledgements}
We thank Xinyang Geng, Justin Gilmer, Alex Kurakin, Jaehoon Lee, Hoang Trieu Trinh, and Greg Yang for useful discussions and feedback. 

\bibliography{cnns}
\bibliographystyle{icml2018}

\normalsize
\onecolumn
\clearpage
\appendix

\begin{center}
\textbf{\large Supplemental Material}
\end{center}
\setcounter{equation}{0}
\setcounter{figure}{0}
\setcounter{table}{0}
\setcounter{page}{1}
\setcounter{section}{0}
\makeatletter
\renewcommand{\theequation}{S\arabic{equation}}
\renewcommand{\thefigure}{S\arabic{figure}}
\renewcommand{\bibnumfmt}[1]{[S#1]}
\renewcommand{\citenumfont}[1]{S#1}

\section{Discussion of Mean Field Theory}
Consider an $L$-layer 1D\footnote{For notational simplicity, as in the main text, we again consider 1D convolutions, but the 2D case proceeds identically.} periodic CNN with filter size $2k+1$, channel size $c$, spatial size $n$, per-layer weight tensors $\omega\in \mathbb{R}^{(2k+1) \times c \times c}$ and biases $b \in \mathbb{R}^{c}$. Let $\phi:\mathbb R \to \mathbb R$ be the activation function and let $h^l_j(\alpha)$ denote the pre-activation at layer $l$, channel $j$, and spatial location $\alpha$. 
Suppose the weights $\omega_{ij}^{l}$ are drawn i.i.d. from the Gaussian $\mathcal N(0, \sigma_\omega^2/(c(2k+1)))$ and the biases $b^{l}_j$ are drawn i.i.d. from the Gaussian $ \mathcal N(0, \sigma_b^2)$.
The forward-propagation dynamics can be described by the recurrence relation,

\begin{equation*}
h^{l+1}_j(\alpha) = \sum_{i\in\chn}\sum_{\beta\in\kn} x^{l}_i(\alpha+\beta) \omega_{ij}^{l+1}(\beta)+b_j^{l+1}\,,\quad  x^l_i(\alpha)  = \phi(h^l_i(\alpha) ).
\end{equation*}
For $l\geq 0$, note that 
(a) $\{h^{l+1}_j\}_j$ are i.i.d. random variables and (b) for each $j$,  $h^{l+1}_j$ is a sum of $c$ i.i.d. random variables with mean zero. The central limit theorem implies that $\{h^{l+1}_j\}_j$ are i.i.d. Gaussian random variables.  
Let $\bm \Sigma^{l+1} =\{\Sigma^{l+1}_{\alpha, \alpha'}\}_{\alpha,\alpha'}$ denote the covariance matrix, where 
$$\Sigma^{l+1}_{\alpha, \alpha'} =\mathbb E\left[ h^{l+1}_j(\alpha) h^{l+1}_j(\alpha') \right],$$
where the expectation is taken over all random variables in and before layer $(l+1)$. Therefore, we have the following lemma.
\begin{lemma} 
As $c\to\infty$, for each $l\geq 0$, $h^{l+1}_j$ is a mean zero Gaussian with covariance matrix $\bm \Sigma^{l+1}$ satisfying the recurrence relation,
\begin{equation}
\bm \Sigma^{l+1} = \mathcal A \star \mathcal C ( \bm\Sigma^l )\,.
\end{equation}
\end{lemma}
\begin{proof}
Let $\theta^l = [W^l, b^l]$ and $\theta^{0:l} = [ \theta^0,\dots, \theta^l]$. Then,

\begin{align}\label{eq_cov}
\Sigma^{l+1}_{\alpha, \alpha'} &= \mathbb E_{\theta^{0:l}}\, \mathbb E_{\theta^{l+1}}\left[ h^{l+1}_j(\alpha)h^{l+1}_j(\alpha')  \right]
\\
&=  
\mathbb E_{\theta^{0:l}}\, \left[  \frac{\sigma^2_\omega}{2k+1}\sum_{\beta\in\kn} \frac 1 {c}\sum_{i=1}^c (x^{l}_i(\alpha+\beta)x^{l}_i(\alpha'+\beta)) + \sigma_b^2\right]\,,
\end{align}
where we used the fact that,
\begin{align*}
\mathbb E_{\theta^{l+1}}\, \left[\omega_{ij}^{l+1}(\beta) \omega_{i'j'}^{l+1}(\beta')\right] = 
\begin{cases}
\frac{\sigma_\omega^2}{c(2k+1)}, \quad &{\rm if }\quad (i, j, \beta) = (i', j', \beta')
\\
\\
0, \quad & {\rm otherwise}. 
\end{cases}
\end{align*} 
Note that $\bm \Sigma^{1}$ can be computed once $h^0$ (or $x^0$) is given.
We will proceed by induction. Let $l\geq 1$ be fixed and assume $\{h_j^l\}_j$ are i.i.d. mean zero Gaussian with covariance $\bm \Sigma^l$. 
It is not difficult to see that $\{h_j^{l+1}\}_j$ are also i.i.d. mean zero Gaussian as $c\to \infty$. To compute the covariance, 
note that for any fixed pair $(\alpha, \alpha')$, $\{x_i(\alpha)x_i(\alpha')\}_{i}$ are i.i.d. random variables. Then,  
\begin{align}\label{eq_lim}
\mathbb E_{\theta^{0:l}} \,\left[ \frac 1 {c}\sum_{i=1}^c (x^{l}_i(\alpha)x^{l}_i(\alpha')) \right] = 
\mathbb E_{\theta^{0:l}}\,\left[x^{l}_i(\alpha)x^{l}_i(\alpha') \right] 
= \mathbb E_{\theta^{0:l}} \,\left[\phi(h^{l}_i(\alpha)) \phi(h^{l}_i(\alpha')) \right] \,.
\end{align}
Thus by eq.~\eqref{eq:cmap},  eq.~\eqref{eq_cov} can be written as,
\begin{align}\label{eq_cov2}
\Sigma^{l+1}_{\alpha, \alpha'}
= \frac{1}{2k+1}\sum_{\beta\in\kn} \mathcal  [ C(\bm \Sigma^l)]_{\alpha+\beta, \alpha'+\beta} \,,
\end{align}
so that,
\begin{equation}
        \bm \Sigma^{l+1} \equiv \mathcal A \star \mathcal C (\bm\Sigma^{l})\,.
\end{equation}

\end{proof} 
The same proof yields the following corollary. 
\begin{corollary}
Let $v = (v_\beta)_{\beta \in{\it ker}}$ be a sequence of non-negative numbers with $\sum_{\beta\in {\it ker}} v_\beta =1$. Let $\mathcal A_v$ be the cross-correlation operator induced by $v$, i.e.,
\begin{equation}
(\mathcal A_v \star f )_{\alpha, \alpha'} = \sum_{\beta\in {\kn}} v_\beta f_{\alpha+\beta, \alpha'+\beta}\,.
\end{equation}
Suppose the weights $\omega_{i,j}^l(\beta)$ are drawn i.i.d. from the Gaussian $\mathcal N(0, \frac{v_\beta} c \cdot \sigma_\omega^2)$. Then the recurrence relation for the covariance matrix is given by,
\begin{equation}
\bm \Sigma^{l+1} = \mathcal A_v\star \mathcal C (\bm\Sigma^l )\,.
\end{equation}
\end{corollary}

\subsection{Back-propagation}
Let $E$ denote the loss associated to a CNN and $\delta^l_j(\alpha)$ denote a backprop signal given by,
$$
\delta^l_j(\alpha) = \frac {\partial E}{ \partial {h^l_j(\alpha)}}. 
$$
The layer-to-layer recurrence relation is given by,
\begin{align*}
\delta^l_j(\alpha) &= \sum_{i\in {\it chn}}\sum_{\alpha'\in{\it sp}} \frac {\partial E}{ \partial {h^{l+1}_i(\alpha')}}\frac{ \partial {h^{l+1}_i(\alpha')}}{ \partial {h^l_j(\alpha)}}
\\
&= \sum_{i\in {\it chn}}\sum_{\beta\in{\it ker}}  \delta^{l+1}_{i}(\alpha -\beta)\omega_{ji}^{l+1}(\beta)\phi'(h_j^l(\alpha)) \,.
\end{align*}
We need to make an assumption that the weights used during back-propagation are drawn independently from the weights used in forward propagation. This implies $\{\delta^l_j\}_{j\in \it chn}$ are independent for all $l$ and for $j\neq j'$,
\begin{align*}
\mathbb E\,[\delta^l_j(\alpha)\delta^l_{j'}(\alpha')] = 0\,,  
\end{align*}
and 
\begin{align*}
\bm\Sigma_{\alpha, \alpha'} &=
\mathbb E\,[\delta^l_j(\alpha)\delta^l_j(\alpha')] 
\\
&= \sum_{i\in {\it chn}}\sum_{\beta\in{\it ker}}      
\mathbb E\,[\delta^{l+1}_i(\alpha -\beta)\delta^{l+1}_i(\alpha' -\beta)]
\mathbb E\,[\phi'(h_j^l(\alpha)\phi'(h_j^l(\alpha')]
\mathbb E\,[\omega_{j,i}^{l+1}(\beta)\omega_{j,i}^{l+1}(\beta)]
\\
& = 
\left(\frac {1} {2k+1}\sum_{\beta\in{\it ker}} 
\bm\Sigma_{\alpha -\beta, \alpha' - \beta}\right) 
\left(\sigma_\omega^2 \mathbb E \, [\phi'(h_j^l(\alpha)\phi'(h_j^l(\alpha')) ]
\right).
\end{align*}
For large $l$, the second parenthesized term can be approximated by $\chi_1$ if $\alpha' = \alpha$ and by $\chi_{c^*}$ otherwise. 

\section{The Jacobian of the {$\mathcal C$-map}}
Recall that 
$\mathcal C:\text{PSD}_n\to\text{PSD}_n$ is given by,
    \begin{equation}\label{eq:cmap1}
    [\mathcal C(\bm \Sigma)]_{\alpha, \alpha'} = \sigma_\omega^2\,\mathbb E_{\bm h \sim\mathcal N(0,\bm \Sigma)} \left[ \phi(h_\alpha)\phi(h_{\alpha'}) \right] +\sigma_b^2\,.
    \end{equation}
We are interested in the linearized dynamics of $\mathcal C$ near the fixed point $\bm\Sigma^*$. Let $J: \mathbb R^{n\times n} \to \mathbb R^{n\times n}$ denote the Jacobian of $\mathcal C$ at $\bm\Sigma^*$. The main result of this section is that $J$ commutes with any diagonal convolution operator. 
\begin{theorem}\label{theorem:commute}
Let $J$ be as above and $\mathcal A$ be any $n\times n$ diagonal matrix and $U$ be any $n\times n$ symmetric matrix. Then,
\begin{equation}
\mathcal A\star J(U) = J (\mathcal A\star U)\,.
\end{equation}
\end{theorem}
Let $\{V_{\alpha, \alpha'}\}_{0\leq \alpha\leq \alpha'\leq n-1}$ be the canonical basis of the space of $n\times n$ symmetric matrices, i.e. $[V_{\alpha, \alpha'}]_{\bar\alpha, \bar\alpha'} = 1$ if $(\alpha, \alpha')=(\bar\alpha, \bar\alpha')$ or $(\bar\alpha', \bar\alpha)$ and 0 otherwise. We claim the following:
\begin{lemma}\label{lemma:representation}
The Jacobian $J$ has the following representation:
\begin{itemize}
\item For the off-diagonal terms (i.e. $\alpha\neq \alpha'$), 
\begin{equation}
    JV_{\alpha,\alpha'} = \chi_{c^*} V_{\alpha,\alpha'}.  
\end{equation}
\item For the diagonal terms,  
\begin{equation}
    JV_{\alpha, \alpha} = \chi_{q^*} V_{\alpha,\alpha} + \kappa \sum_{\alpha'\neq \alpha} V_{\alpha,\alpha'}\,,  
\end{equation}
where $\kappa$ is given by,
\begin{equation}
    \kappa = \frac{\sigma_\omega^2}{2} \mathbb E_{\bm h\sim\mathcal N(0,\bm C^*)} \phi(h_1)\phi''(h_2) ,\; h_1\neq h_2\,. 
\end{equation}
\end{itemize}
\end{lemma}
We first prove Theorem \ref{theorem:commute} assuming Lemma \ref{lemma:representation}, and afterwards we prove the latter. 
\begin{proof}[Proof of Theorem \ref{theorem:commute}]
It is clear that 
$$V_{\text{o.d.}} = {\rm span}\{V_{\alpha, \alpha'}: \alpha\neq \alpha'\}$$ 
is an eigenspace of $J$ with eigenvalue $\chi_{c^*}$. Here ${\rm span}\{X\}$ denotes the linear span of $X$.  
For $\chi_{q^*}\neq \chi_{c^*}$, define,
$$
\tilde V_{\alpha} = V_{\alpha,\alpha} +\frac \kappa {\chi_{q^*} - \chi_{c^*}}\sum_{\alpha'\neq \alpha} V_{\alpha,\alpha'}. 
$$
It is straightforward to verify that 
$$
V_{\text{d}} = {\rm span}\{\tilde V_{\alpha}:\alpha\in\sp\}
$$ 
is an eigenspace of $J$ with eigenvalue $\chi_{q^*}$ and the direct sum $V_{\text{d}} \bigoplus V_{\text{o.d.}}$ is the whole space of $ n\times n$ symmetric matrices.  
Note that $J$ acts on $V_{\text{o.d.}}$ in a pointwise fashion and that $\mathcal A$ maps $V_{\text{o.d.}}$ onto itself (one can form an eigen-decomposition of $\mathcal A$ (and $J$) in $V_{\text{o.d.}}$ using Fourier matrices; see below for details.) 
Thus $\mathcal A$ commutes with $J$ in $V_{\text{o.d.}}$. It remains to verify that they also commute in $V_{\text{d}}$.

,A key observation is that $\{\tilde V_{\alpha}\}_{\alpha\in\sp}$ has a nice group structure,
$$
\{\tilde V_{\alpha}\}_{\alpha\in\sp} = \{V_{\alpha, \alpha}\star\tilde V_{0}\}_{\alpha\in\sp}\,,
$$
which we can use it to form a new basis for $V_{\text{d}}$,
\begin{equation}
V_{\text{d}} = {\rm span}\{ U_\alpha: U_\alpha = F_\alpha \star \tilde V_0, \alpha \in\sp\}\,,
\end{equation}
where $F_\alpha$ is the diagonal matrix formed by the $\alpha$-th row of the $n\times n$ Fourier matrix, i.e. 
$F_\alpha = {\rm diag}((f_{\alpha, \alpha'})_{\alpha'\in\sp})$ with $f_{\alpha, \alpha'} = \frac 1 {\sqrt n} e^{2\pi i \alpha\alpha'\pi/n}$. 
Since each $F_\alpha$ is an eigen-vector of the $2D$ convolutional operator $\mathcal A \star \cdot$  ($\mathcal A$ is diagonal), 
$$
\mathcal A \star (J U_\alpha) = \chi_{q^*} \mathcal A \star  U_\alpha\mathcal = \chi_{q^*} \mathcal A \star F_\alpha \star \tilde V_{0} = \chi_{q^*}  (\mathcal A \star F_\alpha) \star \tilde V_{0} = \chi_{q^*} \lambda_\alpha F_\alpha \star \tilde V_{0} =\chi_{q^*} \lambda_\alpha U_\alpha =  J (\mathcal A \star  U_\alpha)
$$
where $\lambda_\alpha$ is the eigenvalue of $F_\alpha$. This finishes our proof.  
\end{proof}

\begin{proof}[Proof of Lemma \ref{lemma:representation}]
We first consider perturbing the off-diagonal terms. Let $\epsilon$ be a small number and $\alpha\neq \alpha'$. Note that for $(\bar\alpha, \bar\alpha')\notin \{(\alpha, \alpha'), (\alpha', \alpha)\}$,  
\begin{equation}
[\mathcal C (\bm \Sigma^* +\epsilon V_{\alpha, \alpha'})]_{\bar\alpha, \bar\alpha'} =  [\mathcal C (\bm \Sigma^*)]_{\bar\alpha, \bar\alpha'}
\end{equation}
and 
\begin{equation}
[\mathcal C (\bm \Sigma^* +\epsilon V_{\alpha, \alpha'})]_{\alpha, \alpha'} = \sigma_\omega^2\mathbb E \phi(h_1)\phi(h_2) + \sigma_b^2,  
\end{equation}
where $(h_1, h_2)\sim \mathcal N (0, Q)$ with $Q_{11} = Q_{22} = q^*$ and $Q_{12} = Q_{21} = c^*q^* + \epsilon $. Let $c = c^* + \epsilon/q^*$ and choose two independent random variables $u_1$, $u_2\sim\mathcal N (0,1)$. Then,
\begin{equation}\label{equ:001}
[\mathcal C (\bm \Sigma^* +\epsilon V_{\alpha, \alpha'})]_{\alpha, \alpha'} = \sigma_\omega^2\mathbb E \phi(\sqrt{q^*}u_1)\phi(\sqrt{q^*}(c u_1 + \sqrt{1- c^2} u_2))+\sigma_b^2.   
\end{equation}
Taylor expanding the term $\phi(\sqrt{q^*}(c u_1 + \sqrt{1- c^2} u_2))$ about the point $\sqrt{q^*}(c^* u_1 + \sqrt{1- (c^*)^2} u_2)$, one can show,
\begin{equation}
[\mathcal C (\bm \Sigma^* +\epsilon V_{\alpha, \alpha'})]_{\alpha, \alpha'} = c^*q^* + \chi_{c^*} \epsilon + O(|\epsilon|^2)\,,
\end{equation}
which proves the first statement of Lemma \ref{lemma:representation}.  

To prove the second statement, let $\alpha$ be fixed and perturb $\bm \Sigma^*$ by $\epsilon V_{\alpha, \alpha}$. Note that all the terms are unchanged except the ones in the $\alpha$-th row or $\alpha$-th column. 
It is straightforward to show (see~\cite{poole2016}) that  
\begin{equation}
[\mathcal C (\bm \Sigma^* +\epsilon V_{\alpha, \alpha})]_{\alpha, \alpha} = q^* + \chi_{q^*} \epsilon + O(|\epsilon|^2).  
\end{equation}
For any $\alpha'\neq \alpha$, 
\begin{equation}
[\mathcal C (\bm \Sigma^* +\epsilon V_{\alpha, \alpha})]_{\alpha', \alpha} = \sigma_\omega^2\mathbb E \phi(\sqrt{q^*}u_1)\phi(\sqrt{q^*}c^* u_1 + q(\epsilon) u_2))+ \sigma_b^2\,,
\end{equation}
where $u_1$ and $u_2$ are the same as in eq.\eqref{equ:001} and,
\begin{equation}
q(\epsilon) = \sqrt{q^* +\epsilon - (c^*)^2 q^*}\,.
\end{equation}
We can then Taylor expand $q(\epsilon)$ about $\sqrt{q^* - (c^*)^2 q^*}$, $\phi(\sqrt{q^*}c^* u_1 + q(\epsilon) u_2))$ about $(\sqrt{q^*}c^* u_1 + \sqrt{q^* - (c^*)^2 q^*} u_2)$, 
and apply one integration by parts to the second variable (namely, apply the identity $\mathbb E u_1f(u_1) = \mathbb E f'(u_1)$) to find,
\begin{equation}
[\mathcal C (\bm \Sigma^* +\epsilon V_{\alpha, \alpha})]_{\alpha', \alpha} = c^*q^* + \kappa \epsilon + O(|\epsilon|^2).  
\end{equation}
\end{proof}

\section{Construction of Random Orthogonal Kernels}
\subsection{Computational Complexity}
For simplicity, consider constructing a $\Bbbk\times \Bbbk\times c\times c$ orthogonal kernel. The complexity can be roughly determined as follows:
\begin{enumerate}
\item Constructing $O(\Bbbk)$ $c\times c$ symmetric orthogonal matrices takes $O(\Bbbk c^3)$ steps. 
\item For $j=1, \dots, \Bbbk-1$, convolving a $j\times j$ (each entry is a $c\times c$ matrix) matrix with a $2\times 2$ matrix requires $O(j^2)$ matrix multiplications between two $c\times c$ matrices. Since each matrix multiplication costs $O(c^3)$, a total number of $O((\Bbbk c)^{3})$ steps is required for block-wise matrix convolutions.
\item In sum, the computational complexity is about $O((\Bbbk c)^{3})$.
\end{enumerate}

\subsection{Delta Orthogonal Kernels}
  \begin{algorithm}[h!]
    \caption{2D Delta orthogonal kernels for CNNs, available in TensorFlow via the $\operatorname{ConvolutionDeltaOrthogonal}$ initializer.}
    \label{alg:delta}
    \begin{algorithmic}
    \STATE {\bfseries Input:} $\Bbbk$ kernal size, $c_{in}$ number of input channels, $c_{out}$ number of output channels.
    
        {\bf Return: }a $\Bbbk\times \Bbbk \times c_{in}\times c_{out}$ tensor $K$
   
       {\bf Step 1.} Randomly generate a $c_{in}\times c_{out}$ matrix $H$ with orthonormal rows.
       
       {\bf Step 2.} Define a $\Bbbk\times \Bbbk\times c_{in} \times c_{out}$ tensor $K$ in the following way: for $\beta, \beta' $ in $0, 1, \dots, \Bbbk-1$, for $i = 0, \dots, c_{in}-1$, $j=0, \dots, c_{out}-1$, set 
    $$
    K(\beta, \beta', i, j) = 
    \begin{cases}
    H(i, j), &{\rm if }\,\,  \beta =\beta'=[\Bbbk/2] 
    \\
    0,  &{\rm otherwise. }
    \end{cases}
    $$
    \end{algorithmic}
    \end{algorithm}
    
\newpage
\section{Phase Diagram and Vanishing/Exploding gradients}

\subsection{Phase Diagram}

Figure \ref{fig:order-2-chaos} shows the phase diagram derived from the mean field theory of signal propagation in fully-connected networks, reproduced from \cite{pennington2017}. It depicts the ordered and chaotic phases (with vanishing and exploding gradients, respectively) separated by a transition. The variation in value of $q^*$ along the critical line is shown in color. As discussed in the main text, it also applies to the ordered-to-chaotic phase transition of CNNs.
\begin{figure}[t]
\centering
\begin{minipage}{.42\textwidth}
  \centering
  \vspace{-0.5cm}
  \includegraphics[width=0.8\textwidth]{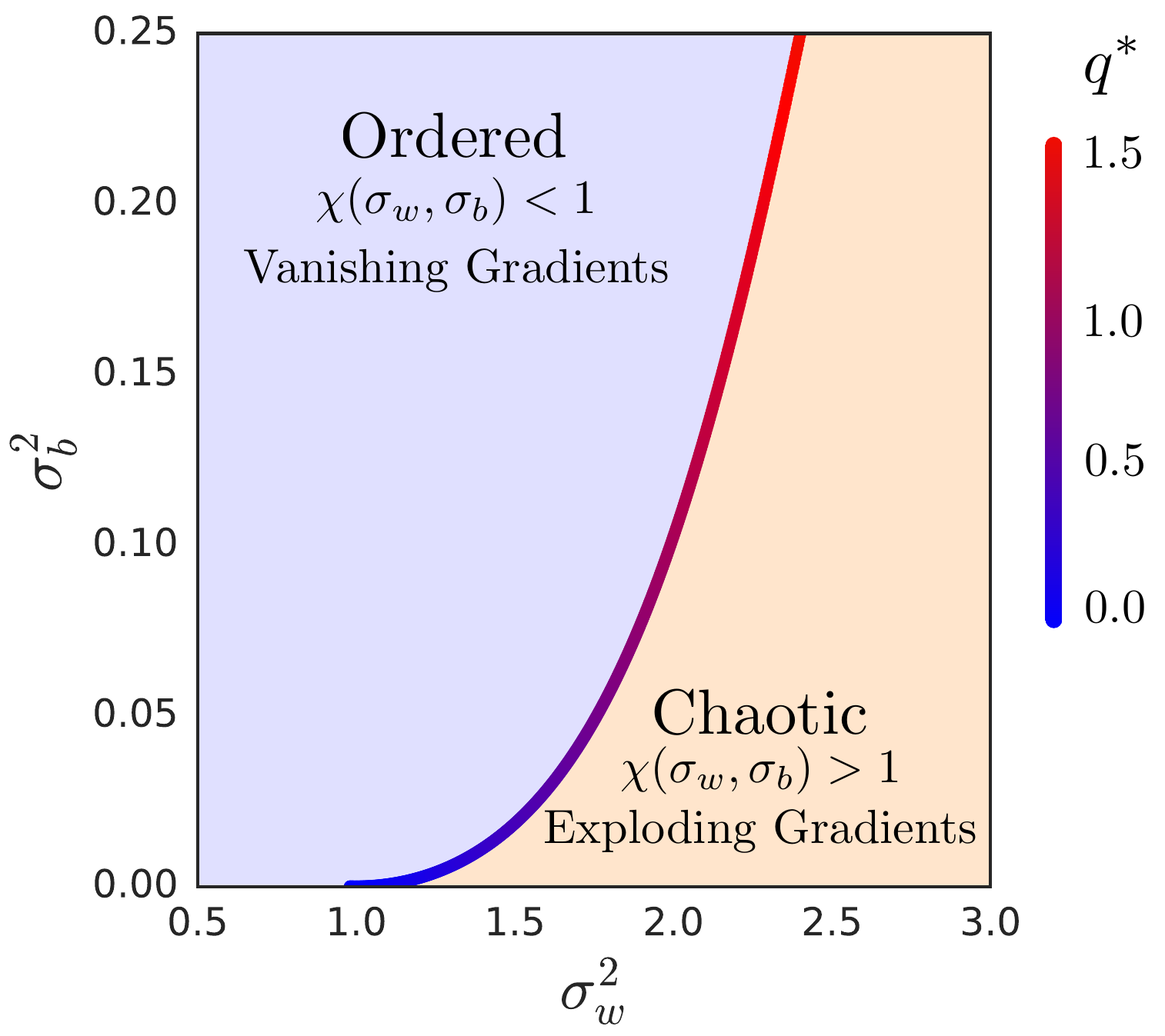}
\caption{Phase diagram for fully-connected networks (reproduced from \cite{pennington2017}). In our analysis, we find they also apply to CNNs.}
\label{fig:order-2-chaos}
\end{minipage}%
\hspace{1.5cm}
\begin{minipage}{.42\textwidth}
  \centering
  \vspace{0.3cm}
  \includegraphics[width=0.90\textwidth]{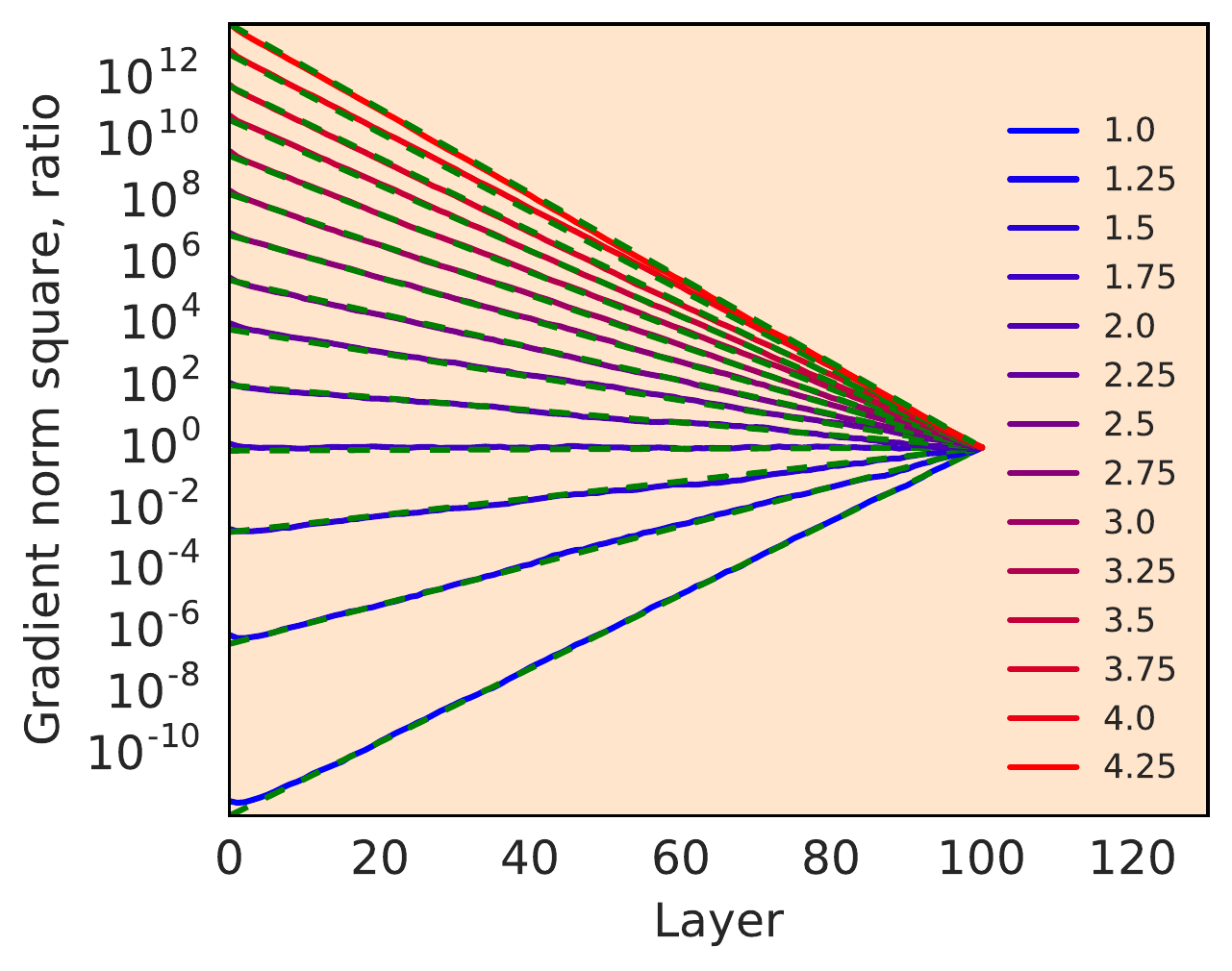}
  \vspace{-0.05cm}
\caption{Vanishing/exploding gradients in a random CNN as the weight variance $\sigma^2_w$ is swept so as to transition from the ordered phase (blue curves) to the chaotic phase (red curves). Mean field theory predicts an exponential decay/growth, which is overlaid (green, dashed).}
\label{fig:gradVanish}
\end{minipage}
\end{figure}

\subsection{Vanishing and Exploding Gradients}

Figure \ref{fig:gradVanish} depicts the behavior of gradients in an $L = 100$ layer deep random CNN as $\sigma^2_w$ is varied across the phase boundary, from the ordered to the chaotic phase. In this case, the input size $n = 10$, kernel size $2k+1 = 3$, and number of channels $c=2000$. The input was synthetic and generated i.i.d from normal distributions with a random spatial covariance matrix.  The $y$-axis plots the squared norm of the gradient with respect to the weights in layer $\ell$, $||\nabla_{\bm W^{\ell}} h^{L} ||^{2}_{2}$, with $\ell = 0, ..., 100$, relative to the last layer gradient, $||\nabla_{\bm W^{L}} h^{L} ||^{2}_{2}$. The bias variance is fixed, $\sigma^2_b = 0.05$, while the weight variance is swept from $\sigma^2_w = 1.0$ (blue curves), where gradients vanish exponentially as a function of layer distance $L - \ell$ from the output, to $\sigma^2_w = 4.25$ (red curves), where gradients explode exponentially. The mean field theory prediction (overlaid in dashed green) gives excellent agreement with the empirical result. 

\section{Distribution of singular values of weight matrices}
\label{}
The end-to-end Jacobian $\bm J$ depends on the matrix of weights $\bm W^{l}$, and the singular value distribution of the latter plays a key role, as discussed in the main text. Figure \ref{fig:SVD_weights} compares the singular value distribution of weight matrices $\bm W^{l}$ in the convolutional vs. fully-connected setting. In more detail, $\bm W^{l}$ in the convolutional case can be considered an $n \times n$ circulant tiling of $c \times c$ dense blocks, where each matrix element is generated i.i.d. from $\mathcal {N}(0, 1/(c(2k+1)))$.
For fixed $n=26$ 
and $2k+1=5$ 
we compute the singular value distribution, as $c$ increases, for single draws. This is compared to the distribution for the weight matrix in the fully-connected setting, obtained from a dense $nc \times nc$ matrix $\bm W^{l}$ whose entries are drawn i.i.d from $\mathcal{N}(0, 1/(nc))$. We empirically find the agreement between the two improves as the channel number increases, suggesting that the random matrix theory analysis of \cite{pennington2017, PenningtonSG18} carries over to the convolutional setting. 

\section{Multiple depth scales in signal propagation}
\label{sec:SM_depthscale}

Figure \ref{fig:Fourier} empirically demonstrates the existence of multiple depth scales, as discussed in Section \ref{sec:depth}. We consider an ensemble of random CNNs and compute the average covariance matrix $\bm \Sigma^{l}$ as a function of depth. We consider networks with $\erf$ nonliearities with $\sigma_w=\frac{3}{2}$ and $\sigma_b=\frac{1}{2}$ applied to $1$D images of size $n=10$. The initial data covariance $\bm \Sigma^{0}$ is chosen so that $\bm \epsilon^{0} = \bm \Sigma^* - \bm\Sigma^{0}$ is small and has an off-diagonal structure. In particular, all entries of $\bm \epsilon^{0}$ except the first cyclic diagonal entries are taken to be zero, and that diagonal has Fourier transform given by $-\frac{1}{6}[1,\frac{2}{3},(\frac{2}{3})^3,(\frac{2}{3})^5,(\frac{2}{3})^4,(\frac{2}{3})^2,(\frac{2}{3})^4,(\frac{2}{3})^5,(\frac{2}{3})^3,\frac{2}{3}]$. We used a spatially non-uniform kernel of size $2k+1=3$, with weights $v = [0.025, 0.950, 0.025]$. We then averaged $\bm \Sigma^{l}$ over this ensemble of networks to construct $\bm \epsilon^{l}$. By decomposing the vector of first cyclic diagonal entries into Fourier modes, we can observe how the signal decays differently along different modes. Figure \ref{fig:Fourier} plots the absolute value of the coefficient of the Fourier decomposition as a function of depth. Our mean field theory predictions for the different depth scales are in excellent agreement with the empirical simulations.

\begin{figure}[t!]
\centering
\begin{minipage}{.42\textwidth}
  \centering
  \includegraphics[width=0.75\textwidth]{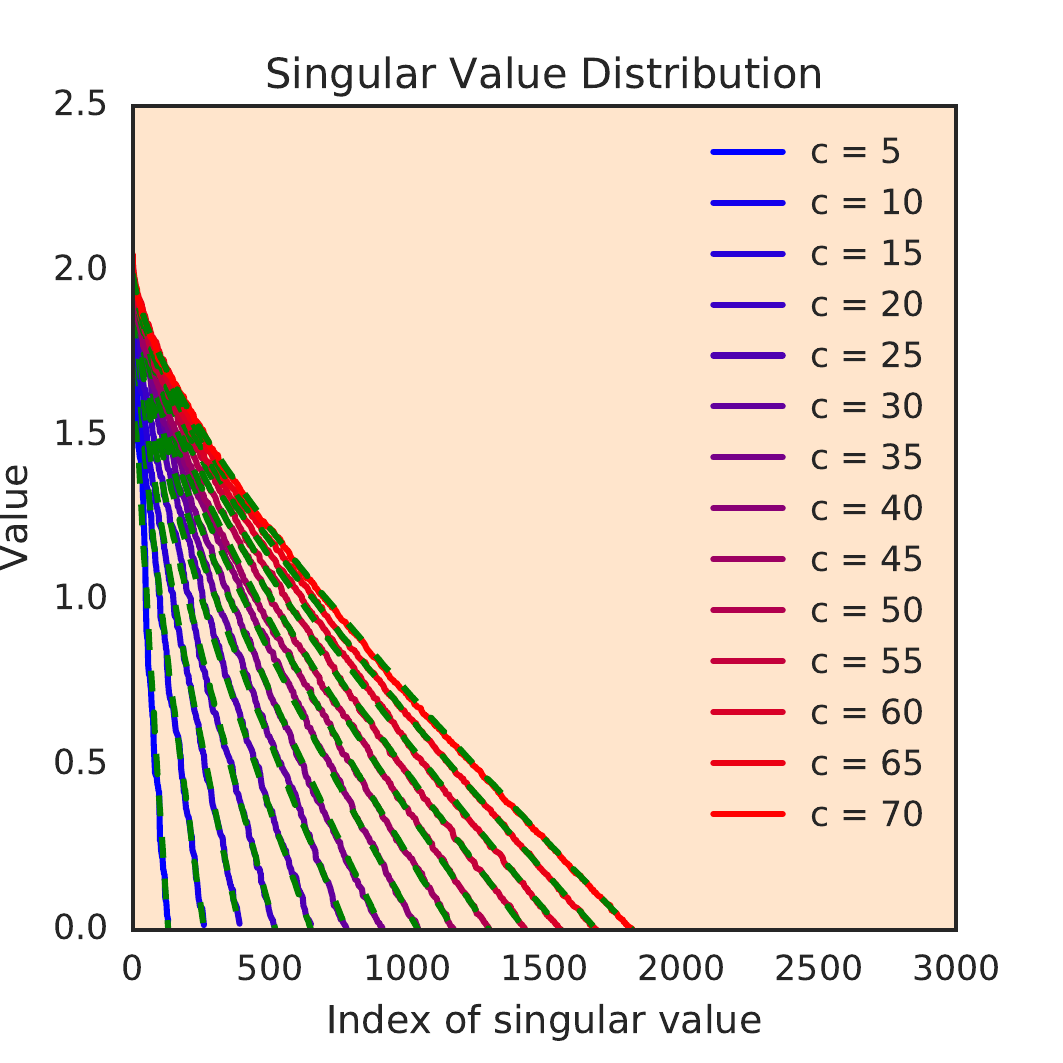}
\caption{A comparison of the singular value distribution of $\bm W^l$ for a block-circulant matrix of shape $nc \times nc$ with i.i.d entries (solid) against that of an $nc \times nc$ dense matrix (overlaid in dashed green), for fixed spatial width $n=26$ and kernel size $2k+1=5$. While there are discrepancies for very small $c$ that are not visible on this scale, there is good agreement as the channel size $c$ increases (blue to red curves).}
\label{fig:SVD_weights}
\end{minipage}%
\hspace{1.5cm}
\begin{minipage}{.42\textwidth}
  \centering
  \vspace{-0.35cm}
  \includegraphics[width=0.8\textwidth]{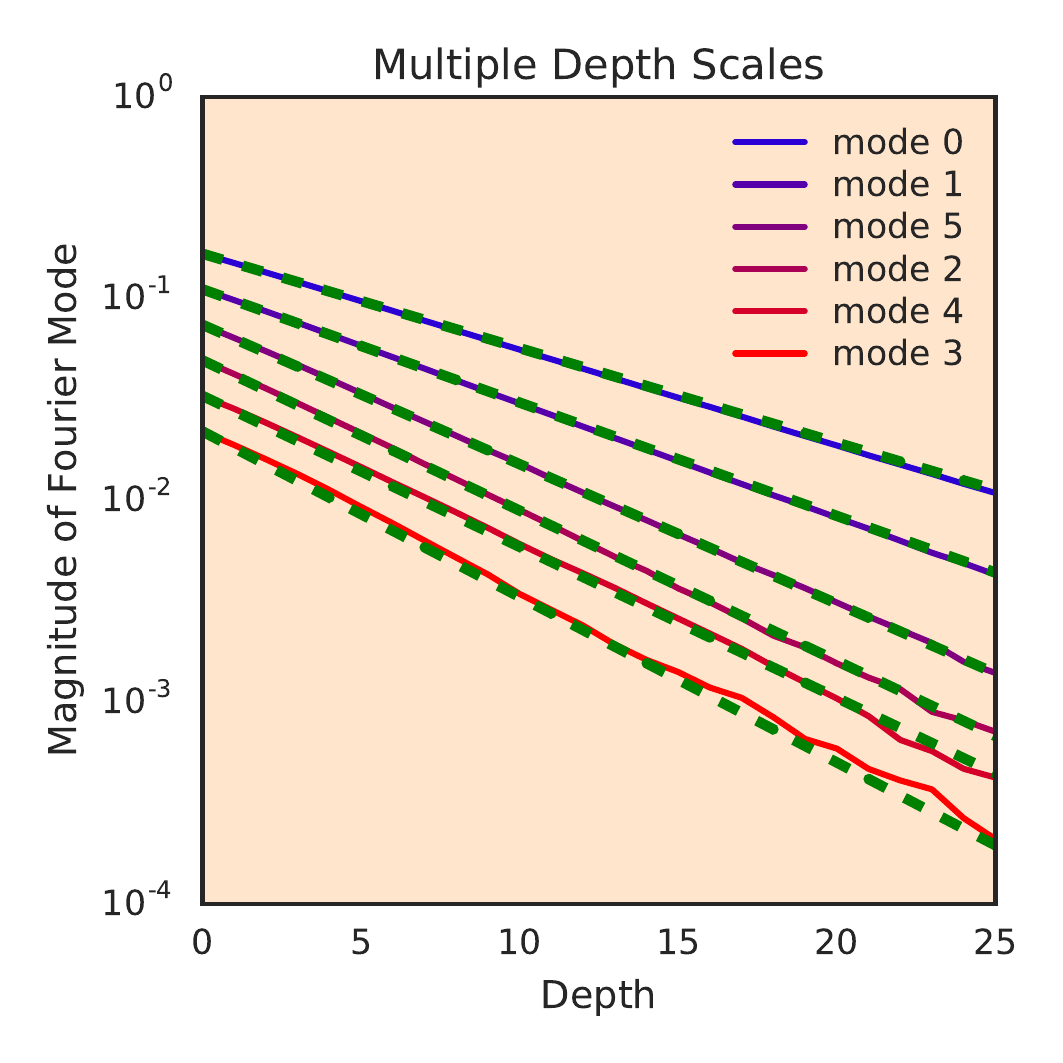}
  \vspace{-0.33cm}
\caption{The existence of multiple depth scales in random CNNs, with a comparison between empirical results (colored) and theoretical predictions (dashed green). The depth scales are reflected in the differing slopes of the curves, with the zero frequency mode decaying most slowly. See Section~\ref{sec:SM_depthscale} for a detailed description of the experiment.}
\label{fig:Fourier}
\end{minipage}
\end{figure}


\end{document}